\documentclass[twoside]{article}

\usepackage[accepted]{aistats2021}
\usepackage{biblatex}
\usepackage{lipsum}

\usepackage{microtype}
\usepackage{graphicx}
\usepackage{booktabs} 
\usepackage{amsfonts}
\usepackage{wrapfig}
\usepackage{lipsum}
\usepackage{amsmath, amssymb, amsthm}
\usepackage{bm}
\usepackage{algorithm}
\usepackage{algorithmic}
\graphicspath{{figs/}}
\usepackage{caption}
\usepackage{subcaption}
\usepackage{enumitem}
\usepackage{wrapfig}
\usepackage{hyperref}

\makeatletter
  {\par\vskip\textfloatsep\vfil}
\makeatother

\newtheorem{theorem}{Theorem}
\newtheorem{othertheorem}{othertheorem}[section]
\newtheorem{lemma}[othertheorem]{Lemma}

\newtheorem{fact}[othertheorem]{Fact}

\newcommand{\norm}[1]{\lVert#1\rVert}

\newcommand{\mat}{\mathbf}

\newcommand{\unif}{\mathrm{unif}}

\newcommand{\vect}[1]{\mathbf{#1}}

\newcommand{\expect}{\mathbb{E}}

\newcommand{\indict}{\mathbb{I}}

\newcommand{\relu}[1]{\sigma\left(#1\right)}

\numberwithin{equation}{section}

\newcommand{\w}{\bm{w}}

\begin{document}
\twocolumn[

\aistatstitle{DebiNet: Debiasing Linear Models with Nonlinear Overparameterized Neural Networks}

\aistatsauthor{ Shiyun Xu$^\star$ \And Zhiqi Bu$^\star$}

\aistatsaddress{ Department of Applied Mathematics and Computational Science
\\
University of Pennsylvania } ]

\begin{abstract}
	Recent years have witnessed strong empirical performance of over-parameterized neural networks on various tasks and many advances in the theory, e.g. the universal approximation and provable convergence to global minimum. In this paper, we incorporate over-parameterized neural networks into semi-parametric models to bridge the gap between inference and prediction, especially in the high dimensional linear problem. By doing so, we can exploit a wide class of networks to approximate the nuisance functions and to estimate the parameters of interest consistently. Therefore, we may offer the best of two worlds: the universal approximation ability from neural networks and the interpretability from classic ordinary linear model, leading to both valid inference and accurate prediction. We show the theoretical foundations that make this possible and demonstrate with numerical experiments. Furthermore, we propose a framework, DebiNet, in which we plug-in arbitrary feature selection methods to our semi-parametric neural network. DebiNet can debias the regularized estimators (e.g. Lasso) and perform well, in terms of the post-selection inference and the generalization error.
\end{abstract}

\vspace{-0.25cm}
\section{Introduction}
\vspace{-0.15cm}
In the studies including signal inverse problem, genome-wide association studies, criminal justice and economic forecasts, linear models may be preferred over non-parametric models such as neural networks, due to their simplicity and interpretability. Especially, the ordinary least squares (OLS) estimator is known to be consistent and thus capable of offering both valid inference as well as prediction. In high dimension though, the unique OLS is not available and certain structural assumptions such as the sparsity has to be introduced, e.g. via the regularization. However, such regularization induces bias to the estimators and many efforts have been devoted to correcting or debiasing this bias for post-selection inference or selective inference \cite{taylor2015statistical,berk2013valid,lee2016exact,tibshirani2016exact}. Yet the methods are usually specific to a type of estimator and hard to compute efficiently. On the other hand, the black-box neural networks commonly have much stronger prediction performance over linear models, by adapting to the features automatically without requiring many assumptions. Nevertheless, it can be difficult to explain the prediction of a neural network and therefore to trust it without the inference guarantee.

The need to bridge the gap between the inference and the prediction motivates the partially linear model (PLM) \cite{engle1986semiparametric,liang1999estimation,chen1988convergence,hardle2012partially}, a semi-parametric model that combines the strengths of two lines of researches. We leverage the approximation power of the non-parametric component to efficiently and consistently estimate the significant regressors in the linear component, thus allowing the regularized linear models to infer. Additionally, the non-parametric component trades its interpretability off and contributes to the prediction power of the PLM.

Mathematically, PLM is a generalized problem which includes the linear problem as a sub-case: 
\begin{align}
\mathbf{y}=\mathbf{D}\bm\beta+f(\mathbf{Z})+\bm\epsilon.
\label{eq:PLM form}
\end{align}
Here $\mathbf{D}$ and $\mathbf{Z}$ are two data matrices, $\mathbf{y}$ is the label, $\bm\beta$ is the parameters of interest, $f$ is a (possibly) non-linear function and $\bm\epsilon$ is the noise. The linear component $\mathbf{D}\bm\beta$ is parametric and the coefficients $\bm\beta$ can be estimated consistently \cite{hardle2012partially,robinson1988root}, provided that the non-linear component $f$, which is non-parametric, can be accurately approximated. In other words, the effects of estimation and approximation in PLM are separated into different components.

The key element to the strong PLM performance is the approximation power of the non-parametric component. Traditionally, kernel estimators such as Nadaraya-Watson (NW) kernel \cite{robinson1988root} or Local Linear (LL) regression \cite{hamilton1997local} are applied for the approximation and have been proven to be consistent. Recently, Double/Debiased Machine Learning (DML) is proposed in \cite{chernozhukov2018double} that uses any qualified machine learning models to replace the kernels. In this work, we propose to apply a particular choice of models, namely the over-parameterized neural network, to extend the DML framework both theoretically and algorithmically. To be more specific, our contribution is three-fold:
\begin{enumerate}[leftmargin=0.8cm,labelsep=0.3cm]
    \item We propose a PLM that employs arbitrary over-parameterized neural networks to approximate nuisance functions, based on the partialling-out technique. We illustrate its advantages over the computational efficiency, the flexibility, the robustness and the performance.
    \item We utilize the over-parameterized neural network theories, such as the Neural Tangent Kernel (NTK) \cite{jacot2018neural}, to guarantee the goodness of approximation and the consistency of $\bm\beta$ estimators.
    \item We further develop a framework, DebiNet, that uses our PLM to debias arbitrary feature selection methods and demonstrates promising results on both the inference and the prediction.
\end{enumerate}
Moreover, while our PLM can be applied to real-world datasets, as demonstrated in Section \ref{sec:7}, we focus on synthetic data for Table \ref{table1:compare with PLMs} and Table \ref{table2:debiasing} as we need the access to true $\bm\beta$ to illustrate the statistical consistency, which is our main goal of debiasing. 

\begin{figure}[!htb]
\vspace{-0.5cm}
\centering
\includegraphics[width=12cm]{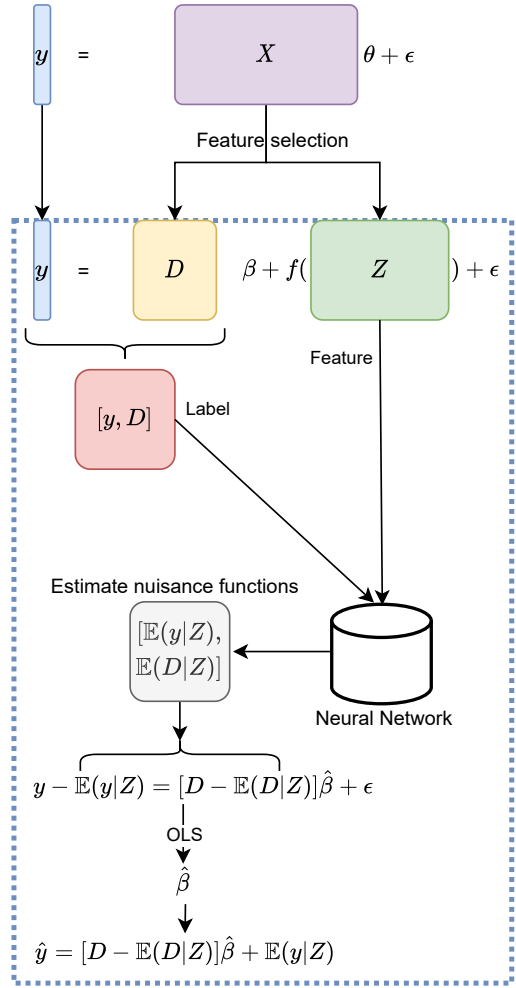}
\vspace{-2cm}
\caption{The architecture of DebiNet. The DebiNet applies arbitrary feature selection method and then PLM-NN in the blue dashed frame to estimate the parameters of interest as well as to predict.}
\end{figure}

\vspace{-0.3cm}
\subsection{Related Work}
\vspace{-0.1cm}
\paragraph{Feature selection and post-selection inference}
In the high dimensional data analysis, many feature selection methods such as the forward-backward selection \cite{efroymson1960multiple}, Lasso \cite{tibshirani1996regression}, adaptive Lasso \cite{zou2006adaptive}, elastic net \cite{zou2005regularization}, (sparse) group lasso \cite{simon2013sparse,friedman2010note} and SLOPE \cite{bogdan2015slope} have been proposed to select important variables. However, these methods may incur large bias and make valid inference difficult \cite{park2008bayesian,hans2009bayesian} if not impossible. A long list of researches have developed methods to debias these models, especially the Lasso \cite{lee2016exact,chernozhukov2015post,taylor2014post}, by separating the procedure of variable selection and coefficient estimation. Two well-known yet quite different approaches are the OLS post-Lasso \cite{belloni2013least} and the debiased Lasso (or the desparsified Lasso) \cite{zhang2014confidence,van2014asymptotically,li2017debiasing}. Interestingly, both methods are deeply related to PLM.

\vspace{-0.25cm}
\paragraph{Partially linear models}
PLMs are commonly used in multiple fields of study to partly debias the estimators. Kernel smoothing \cite{speckman1988kernel}, local polynomial regression \cite{hamilton1997local}, spline-based regression \cite{heckman1986spline} and related techniques have been applied to PLM and proven consistent under various conditions. To take one step further, Double/Debiased Machine Learning (DML) \cite{chernozhukov2018double} proposed to use sample-splitting and any qualifying machine learning model to learn PLM, but the conditions under which a model qualifies can be case-specific and hence hard to verify. We note that the Lasso, decision trees, random forests and an under-parameterized two-neuron neural network are investigated in \cite{chernozhukov2018double}, yet the over-parameterized neural networks are not covered. We emphasize that the choice of model is critical and our proposal is efficient in learning with theoretical support. 

\vspace{-0.25cm}
\paragraph{Over-parameterized neural networks}
To solve complicated machine learning tasks, state-of-the-art neural networks are heavily over-parameterized. Such neural networks may enjoy many nice properties. For example, the universal approximation property \cite{stinchcombe1989universal,gybenko1989approximation,white1990connectionist,hornik1991approximation,leshno1993multilayer} allows neural networks to learn arbitrary target functions. Recently, NTK has been explored to study the generalization behavior and the convergence of over-parameterized neural networks \cite{arora2019fine,allen2019learning,du2018gradient,lee2019wide,du2018gradient2}. In this work, we follow this approach to show that multivariate-output neural networks converge to the global minimum exponentially fast.


\vspace{-0.25cm}
\section{Setup and Notation}
\vspace{-0.15cm}
Suppose the data matrix $\mathbf X\in\mathbb{R}^{n\times p}$ with i.i.d. observations, true parameters $\bm\theta\in\mathbb{R}^p$, label $\mathbf y\in\mathbb{R}^n$ and $\bm\epsilon$ is the random noise. We consider the linear model
\begin{align}
    \mathbf y=\mathbf X\bm\theta+\bm\epsilon
    \label{eq:original}
\end{align}
and apply a feature selection method which generates an estimator $\hat{\bm\theta}$ together with an active set of $[p]$ defined as $S:=\{j: \hat \theta_j\neq 0\}$. We rewrite $\mathbf X=[\mathbf D,\mathbf Z]$ by grouping the selected features into the submatrix $\mathbf D:=\mathbf{X}_S\in\mathbb{R}^{n\times p_L}$ and the unselected ones into $\mathbf Z:=\mathbf{X}_{S}^C\in\mathbb{R}^{n\times p_N}$. We can write $\bm\theta^\top=[\bm\beta^\top,\bm\gamma^\top]$ according to whether a feature is selected or not. Then, the original linear model \eqref{eq:original} is equivalent to
\vspace{-0.2cm}
\begin{align}
 \mathbf y=\mathbf D\bm\beta+\mathbf Z\bm\gamma+\bm\epsilon
 \label{eq:true model}
\end{align}
in which we assume $\mathbb{E}(\bm\epsilon|\mathbf D,\mathbf Z)=0$ and Var$(\bm\epsilon)=\sigma_\epsilon^2\mathbf{I}$, same as in \cite{robinson1988root} and \cite[Example 1.1]{chernozhukov2018double}. Here we can consider \textbf{arbitrary feature selection methods} including the arguably most well-known $\ell_1$ regularized linear model, Lasso, as our main example: 
\vspace{-0.15cm}
\begin{align*}
\bm{\hat\theta}_\text{Lasso}=\text{argmin}_{\bm\theta} \frac{1}{2}\|\mathbf X \bm\theta-\mathbf y\|^2+\lambda\|\bm\theta\|_1
\end{align*}
We generalize the model \eqref{eq:true model} to a partially linear model as in \eqref{eq:PLM form}, in which $\mathbf D\bm\beta$ is the parametric component and $f(\mathbf Z)$ is the non-parametric component with $f:\mathbb{R}^{p_N}\to\mathbb{R}$. 

On one hand, this formulation enables us to incorporate non-parametric tools in estimating the parameters of interest $\bm\beta$ consistently. Notice that $f(\mathbf Z)=\mathbf Z\bm\gamma$ is only estimated as a whole without estimating $\bm\gamma$. In other words, we trade off the consistency of estimating $\bm\gamma$ to gain the consistency of estimating $\bm\beta$. When the feature selection method is effective, most of the significant features are selected, together with some noisy features. In this case, the partially linear model with carefully chosen nuisance functions and tuned hyperparameters, is expected to estimate the coefficients of the selected features consistently.

The trick of obtaining consistent estimator is to apply the conditional expectation on $\mathbf Z$ to orthogonalize or `\textbf{partial out}' the non-linear component:
\vspace{-0.1cm}
\begin{align*}
\mathbf y=\mathbf D\bm\beta+f(\mathbf Z)+\bm\epsilon
\Longrightarrow
\mathbb{E}(\mathbf y|\mathbf Z)=\mathbb{E}(\mathbf D|\mathbf Z)\bm\beta+f(\mathbf Z)
\end{align*}
Taking the difference between the two equations leaves an ordinary linear model without intercept,
\vspace{-0.1cm}
\begin{align}
\mathbf y-\mathbb{E}(\mathbf y|\mathbf Z)&=(\mathbf D-\mathbb{E}(\mathbf D|\mathbf Z))\bm\beta+\bm\epsilon
\label{eq:truebeta}
\end{align}
Therefore, as long as we can accurately estimate the conditional expectation (at least asymptotically), $\bm{\hat\beta}$ is consistent by the theory of ordinary least squares. On the other hand, PLM can accurately predict on future data, if the conditional expectation is well-approximated:
\vspace{-0.1cm}
\begin{align}
\mathbf{\hat y}=\mathbf D\bm{\hat\beta}+\hat f(\mathbf Z).
\end{align}
Here $\hat f(\mathbf Z)$ is an approximation of the unknown $f(\mathbf Z)$.

\vspace{-0.25cm}
\section{Partially Linear Model with Neural Network}
\vspace{-0.15cm}
We propose an efficient PLM in Algorithm \ref{alg:PLMNN}, using the neural network to universally approximate the mapping $\mathcal{M}(\mathbf{Z}):=\mathbb{E}([\mathbf{y},\mathbf{D}]|\mathbf{Z})$ in a single fitting, without approximating $f$ explicitly. Comparing with other PLMs, our method enjoys several advantages over the computational efficiency (fewer model fittings and less memory burden), the flexibility with multivariate output and network architectures, as well as the provable global convergence at linear rate.
\setlength{\intextsep}{1\baselineskip}
\begin{algorithm}[!htp]
\centering 
\caption{Partially Linear Model with Neural Networks (PLM-NN)}
\begin{algorithmic}
	\STATE{\textbf{Input}: Data matrix $[\mat{D},\mat{Z}]$, label $\vect{y}$} \\
	\STATE{\textbf{Estimation of $\bm\beta$:}}
	\\
	\STATE{\quad\quad 1. fit $[\vect{y},\mat{D}]\sim \mat{Z}$ via over-parameterized neural network to derive $\mathbb{E}([\vect y,\mat D]|\mat Z)$;}
	\\
	\STATE{\quad\quad 2. fit $\vect y-\mathbb{E}(\vect y|\mat Z)\sim \mat D-\mathbb{E}(\mat D|\mat Z)$ via OLS to derive $\bm{\hat\beta}$;}
	\\
	\STATE{\textbf{Prediction of $\vect y$ and Estimation of $f$:}}
	\\
	\STATE{\quad\quad 3. define $\vect {\hat y}:=\mathbb{E}(\vect y|\mat Z)+(\mat D-\mathbb{E}(\mat D|\mat Z))\bm{\hat\beta}$ and $\hat f(\mat Z):=\mathbb{E}(\vect y|\mat Z)-\mathbb{E}(\mat D|\mat Z)\bm{\hat\beta}$}
\end{algorithmic}
\label{alg:PLMNN}
\end{algorithm}

Historical researches and recent advances have been made towards the fast and accurate estimation of PLMs. Traditional PLMs (see Algorithm \ref{alg:plm kernel} in Appendix) employ kernel methods (e.g. the NW kernel \cite{robinson1988root} and the LL estimator \cite{hamilton1997local}) to estimate the nuisance functions $m_y(\mathbf{Z}):=\mathbb{E}(\mathbf{y}|\mathbf{Z})$ and $m_D(\mathbf{Z}):=\mathbb{E}(\mathbf{D}|\mathbf{Z})$ separately, before fitting the OLS to obtain $\bm{\hat\beta}$. Kernels with proper bandwidths has been rigorously shown to consistently estimate the nuisance functions \cite{robinson1988root,hardle2012partially,li1996root}, under mild conditions, including the case that $\mathbf{D}$ is multivariate.

However, additional models may need to be fitted to estimate $\bm\beta$ and to predict $\mathbf{\hat y}$. As a toy example, consider $\mathbf{y}=\mathbf{D}\bm\beta+\mathbf{Z}\bm\gamma$ and $\mathbf{D}$ is independent of $\mathbf{Z}$. Due to the dependence conditions, different bandwidths $h_D, h_y$ are used in the weight functions: $h_y$ is finite for approximating $m_y$ but $h_D$ should be set to infinity for approximating $m_D$. As a consequence, another possibly different bandwidth $h_f$ is needed for approximating $f$ to predict $\mathbf{\hat{y}}$, resulting in a total of three kernel regressions and one OLS regression. In addition, kernel regressions require memorizing the entire dataset and incur severe storage issues on big data.

Recently, DML \cite{chernozhukov2018double} extends PLMs, by replacing the kernel regression with a rich class of qualified machine learning methods and by employing sample-splitting to remove the bias from overfitting nuisance functions. 
Unfortunately, checking the qualification of models may be difficult and the sample-splitting adds to the computational cost. Furthermore, $\mathbf D$ being multivariate renders many popular methods invalid or inefficient to learn $m_D$, including the vanilla Lasso and decision trees. DML suggests to fit each feature of $\mathbf{D}$ to $\mathbf{Z}$ separately so as to only learn univariate output functions. We note that such operation can be costly in high dimension: with $K$-fold sample-splitting, one needs $K(p_L+1)$ model fittings to learn $(m_y,m_D)$.
\begin{table}[!htb]
\centering
\begin{tabular}{|cccc|}
\hline 
PLMs & Estimation & Train& Test
\\
& MSE&MSE&MSE
\\
\hline
\toprule
PLM-NN & 1.59$\times 10^{-5}$ & 10.48 & 10.36 \\
PLM-NW & 0.11$\times 10^{-5}$ & 16.33 & 16.08\\
DML Lasso & 0.60$\times 10^{-5}$ & 15.79 & 15.79 \\
DML DT & 65.1$\times 10^{-5}$ & 336.00 & 342.68\\
DML RF & 0.58$\times 10^{-5}$ & 44.00 & 44.87\\
\bottomrule
\end{tabular}
\captionof{table}{Comparison of PLMs in 50 independent runs. Here PLM-NW denotes the PLM using NW kernel, DT denotes decision trees at depth of 2 and RF denotes random forests over 100 trees. See Appendix \ref{table1:data} for experiment details.}
\label{table1:compare with PLMs}
\end{table}

Instead, we propose to use over-parameterized neural networks as strong candidates compared to other methods. For instance, neural networks may not suffer as much as kernel regressions from the curse of dimensionality \cite{bach2017breaking}, a significant decrease of performance when the dimension increases. We further extend the DML framework by significantly reducing the computational cost as follows. First, we choose the nuisance function $\mathcal{M}: \mathbb{R}^{p_N}\to\mathbb{R}^{p_L+1}$ instead of $(m_y,m_D)$ to handle the multivariate output directly. We will show that our nuisance function $\mathcal{M}$ is learnable to neural networks. Second, we do not need sample-splitting to deal with overfitting, thanks to the universal approximating property of over-parameterized neural networks. In practice, we adopt the early stopping to enhance the generalization performance of our estimator. As a result, the total number of model fittings to approximate nuisance functions reduces from $K(p_L+1)$ to 1. 

We empirically illustrate that our method is comparable to the traditional PLM with kernels and to the original DML with $K$-fold cross-fitting in Table \ref{table1:compare with PLMs}. Here the estimation MSE is $\|\bm{\hat\beta}-\bm\beta\|_2^2/p_L$ (only available in synthetic data) and the train/test MSE is $\|\hat {\mathbf y}-\mathbf y\|_2^2/n$. Notice that throughout this paper, we train the two-layer, fully-connected, ReLU activated neural networks with an appropriately wide hidden layer, for its simplicity in proof. However, the DebiNet can work with any neural networks such as deep fully-connected and convolutional neural networks. We optimize over MSE loss with the gradient descent at a sufficiently small learning rate.

\vspace{-0.3cm}
\section{Gradient Descent Optimizes Over-parameterized Multivariate Networks}
\vspace{-0.15cm}

With the universal approximation capability for any continuous function, the over-parameterized neural networks with optimal parameters can successfully learn $\mathbb{E}(\mathbf{y}|\mathbf{Z})$ and $\mathbb{E}(\mathbf{D}|\mathbf{Z})$. Many works have studied the generalization behavior of the over-parameterized neural networks \cite{novak2018sensitivity,arora2019fine,allen2019learning,cao2019generalization,neyshabur2019role,park2019effect}. We empirically demonstrate that, in our setting of Table \ref{table1:compare with PLMs}, the generalization indeed benefits from the over-parameterization.

The next question is how to efficiently find the optimal parameters and we address this by extending the NTK approach in \cite{du2018gradient}.

To demonstrate the trainability of wide neural networks, we consider a two-layer fully-connected neural networks with rectified linear unit (ReLU) activation. Denoting $\mathbf{W} \in \mathbb{R}^{p_N\times m}, \mathbf{A} \in \mathbb{R}^{m\times(p_L+1)}$ as the weights in first and second layers respectively, we can write the neural network as
\vspace{-0.15cm}
\begin{align}
F(\mat{W},\vect{A}, \vect{z}) = \frac{1}{\sqrt{m}}\sum_{r=1}^{m}\vect{A}_r \relu{\vect{w}_r^\top \vect{z}}
\end{align}
where $F: \mathbb{R}^{p_N}\to\mathbb{R}^{p_L+1}$, $\vect{z} \in \mathbb{R}^{p_N}$ is the input, $\vect{w}_r$ and $\vect{A}_r$ are the weights corresponding to the $r$-th neuron in the hidden layer and $\relu{\cdot}$ is the ReLU activation function. 

Given the dataset $\left\{(\vect{Z}_i,\vect{M}_i)\right\}_{i=1}^n$ with the multivariate response $\vect{M}:=[\vect{y}, \vect{D}]$. We aim to minimize 
\vspace{-0.18cm}
\begin{align}
 L(\mat{W},\vect{A}) = \sum_{i=1}^{n}\frac{1}{2}\left\|F(\mat{W},\vect{A},\vect{Z}_i)-\vect{M}_i\right\|^2. \label{eqn:loss_regression}
\end{align}
Adopting the same strategy as \cite{du2018gradient}, we fix the second layer and apply the gradient descent to optimize the first layer, via the gradient flow defined as
\begin{align*}
&\frac{d\vect{w}_r(t)}{dt} = - \frac{\partial L(\mat{W}(t),\vect{A})}{\partial \vect{w}_r(t)}\\
= &\sum_{h=1}^{(1+p_L)}\sum_{j=1}^n(M_{jh} - F_h(\mat{W},\mat{A},\vect{Z}_j))\frac{\partial F_h(\mat{W},\mat{A},\vect{Z}_j)}{\partial \vect{w}_r}.
\end{align*}
Now we quote an important fact that justifies our main theorem.
\begin{fact}[Assumption 3.1 and Theorem 3.1 in \cite{du2018gradient}]
If for any $i \neq j$, $\vect{Z}_i \not \parallel \vect{Z}_j$, then the least eigenvalue $\lambda_0 := \lambda_{\min}\left(\mat{H}^{\infty}\right) > 0$, where matrix $\mat{H}^\infty \in \mathbb{R}^{n \times n}$ with $(\mat{H})_{ij} ^\infty= \expect_{\vect{w} \sim \mathcal N(\vect{0},\mat{I})}\left[\vect{Z}_i^\top \vect{Z}_j\indict\left\{\vect{w}^\top \vect{Z}_i \ge 0, \vect{w}^\top \vect{Z}_j \ge 0\right\}\right]$.
\label{asmp:main}
\end{fact}
Our main theorem shows that, if the least eigenvalue of $\vect{H}_{sh}(t)$, which will be defined in \eqref{eqn:H_t}, is always lower bounded, then the loss converges to 0 at a linear rate.

\begin{theorem}\label{thm:main_gf}
Suppose the condition of Fact~\ref{asmp:main} holds and for all $i \in [n]$, $\norm{\vect{Z}_i}_2 = 1$  and $\norm{\vect{M}_{i}} \le C$ for some constant $C$.
Then if we set the number of hidden neurons $m = \Omega \left( \frac{(1+p_L)^5n^6}{\delta^3 \lambda_0^4}\right)$ and we i.i.d. initialize $\vect{w}_r \sim N(\vect{0},\mat{I})$, ${A}_{rs} \sim \unif\left\{-1,1\right\}$ for $r \in [m], s \in [1+p_L]$, then with probability at least $1-\delta$ over the initialization, we have 
\begin{align*}
\norm{\vect{F}_s(t)-\vect{M}_s}_2^2 \le \exp(-\lambda_0 t)\norm{\vect{F}_s(0)-\vect{M}_s}_2^2.
\end{align*}
\end{theorem}
\begin{proof}[Proof of Theorem \ref{thm:main_gf} (sketch)]
The full proof can be found in Appendix \ref{AppendixB}. Here we sketch the proof at high level. The conditions of our theorem is to guarantee that $\vect{W}(t)$ and consequently $\vect{H}_{sh}(t)$, although being time-dependent, stay close to their initializations. Such phenomenon is commonly observed and known as the `lazy training' for over-parameterized neural networks. Then a careful analysis on the initialization $\vect{H}_{sh}(0)$ shows that it is close to $\mat{H}^\infty$ and that the NTK is positive definiteness, which leads to the exponentially fast convergence. Interestingly, we note that the NTK is close to a block diagonal matrix $\text{diag}(\mat{H}^\infty,\cdots,\mat{H}^\infty)$, i.e. each dimension of the output evolves under the same dynamics.

More formally, we derive the dynamics of the output by the chain rule and the gradient flow,
\begin{align}
\begin{split}
	\frac{d}{dt} F_{is}(t)	&=\sum_{r=1}^{m}\langle\frac{\partial F_s(\mat{W}(t),\vect{A},\vect{Z}_i)}{\partial \vect{w}_r(t)},\frac{d \vect{w}_r(t)}{dt} \rangle\\
	&= \sum_{h=1}^{(1+p_L)}\sum_{j=1}^{n}(M_{jh}-u_{jh}) (\mat{H}_{sh})_{ij}(t) 
\end{split}
\label{eqn:individual_dynamics}
\end{align}
in which $\mat{H}_{sh}(t)$ is an $n \times n$ matrix:
\begin{align}
(\mat{H}_{sh})_{ij}(t) 
= \sum_{r=1}^{m}\langle\frac{\partial F_s(\mat{W},\vect{A},\vect{Z}_i)}{\partial \vect{w}_r}, \frac{\partial F_h(\mat{W},\mat{A},\vect{Z}_j)}{\partial \vect{w}_r}\rangle \nonumber
=\\ \frac{1}{m} \vect{Z}_i^\top \vect{Z}_j\sum_{r=1}^m A_{rs}A_{rh} \indict\left\{\vect{Z}_i^\top \vect{w}_r(t) \ge 0, \vect{Z}_j^\top \vect{w}_r(t) \ge 0\right\}. 
\label{eqn:H_t}
\end{align}
Vectorizing \eqref{eqn:individual_dynamics} gives
\begin{align*}
\frac{d}{dt} \vect{F}_s(t) = \sum_{h=1}^{(1+p_L)} \mat{H}_{sh}(t)(\vect{M}_s-\vect{F}_s(t))
\end{align*}
which leads to
\begin{align*}
&\frac{d}{dt} (\vect{M}_s- \vect{F}_s(t))= -\sum_{h=1}^{(1+p_L)} \mat{H}_{sh}(t)(\vect{M}_s-\vect{F}_s(t))
\end{align*}
This matrix ordinary differential equation has a solution which decays exponentially fast (see Figure \ref{fig:NTK loss}), provided that $\vect{H}_{sh}(t)$ is positive definite with the least eigenvalue bounded away from 0.
\end{proof}

\begin{figure}[!htb]
\vspace{-0.66cm}
\includegraphics[width=0.45\textwidth]{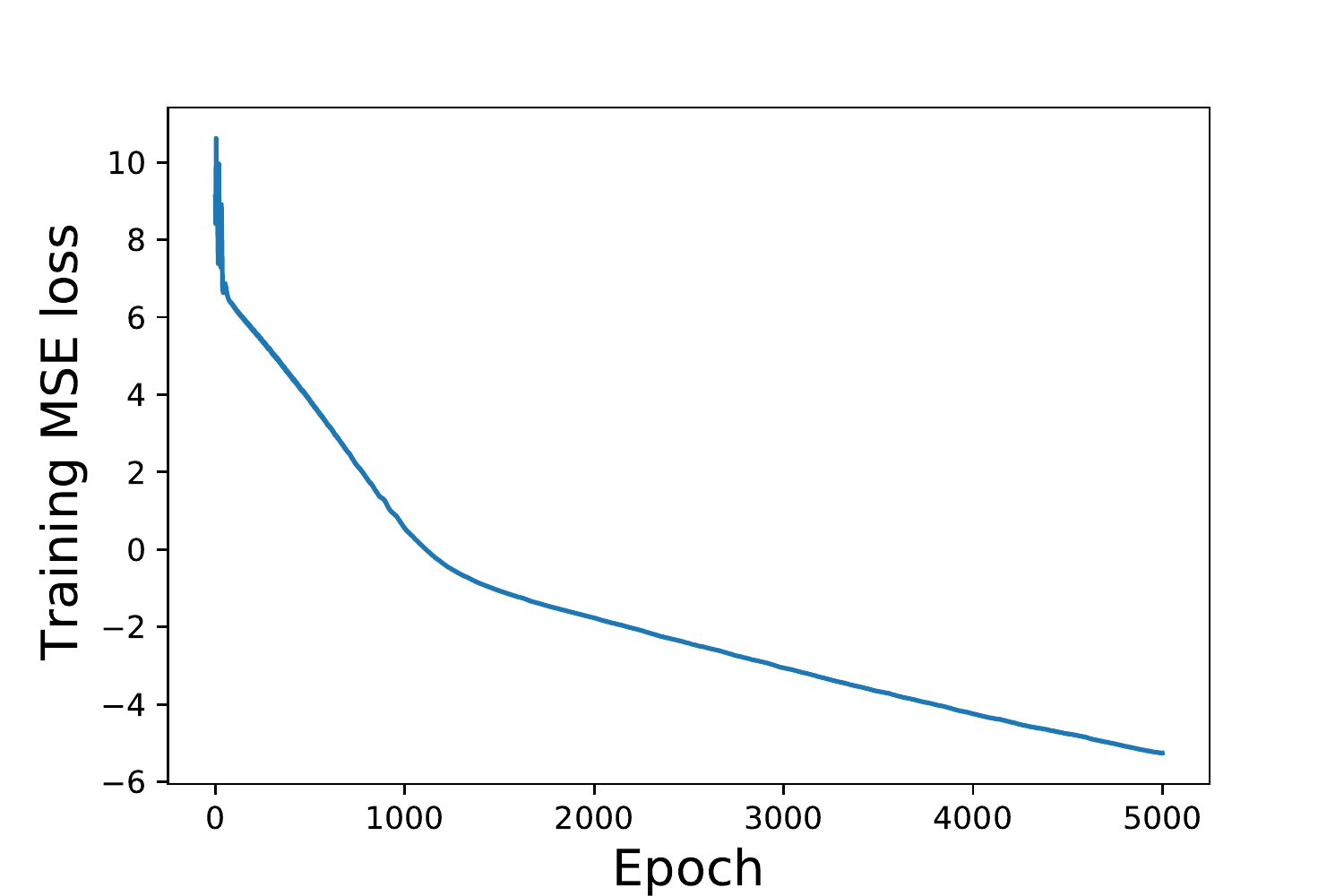}
\vspace{-0.2cm}
\captionof{figure}{Same setting as Table \ref{table1:compare with PLMs} except $n=100$ and $\mathbf{Z}$ is normalized. The loss is in logarithmic scale.}
\label{fig:NTK loss}
\end{figure}

Similar to the analysis in \cite{du2018gradient}, we remark that our analysis can be easily generalized from the continuous time analysis to the discrete time one, as well as to training both layers jointly. To see this, we train both layers of the same network as in Table \ref{table1:compare with PLMs} with the gradient descent. The linearity of the log training loss in Figure \ref{fig:NTK loss} confirms the exponential convergence rate.

\vspace{-0.25cm}
\section{Consistency of $\hat{\bm\beta}$}
\vspace{-0.15cm}
Denote $\mathcal{X}:=\mathbf D-\mathbb{E}(\mathbf D|\mathbf Z)$ and $\mathcal{Y}:=\mathbf y-\mathbb{E}(\mathbf y|\mathbf Z)$. If the nuisance function $\mathcal{M}$ is consistently estimated, then the standard OLS theory states that the OLS estimator $\bm{\hat\beta}=(\mathcal{X}^\top\mathcal{X})^{-1}\mathcal{X}^\top\mathcal{Y}$ is $\sqrt{n}$-consistent to $\bm\beta$, as $\sqrt{n}(\bm{\hat\beta}-\bm\beta)$ converges in probability. Though in reality, the errors in approximating the conditional expectation, incurred when learning $\mathcal{M}$, are unavoidable for any model including the neural networks. Thus they may cause $\bm{\hat\beta}$ to be inconsistent and the bias needs careful adjustment by the measurement error model (or errors-in-variables) theory as follows.

Suppose instead of $(\mathcal{X},\mathcal{Y})$, we only observe data $(\widetilde{\mathcal{X}},\widetilde{\mathcal{Y}}):=({\mathcal{X}}+\bm\epsilon_X,{\mathcal{Y}}+\bm\epsilon_Y)$ which are measured with independent errors $\bm\epsilon_X, \bm\epsilon_Y$. Then the estimator is
\begin{align}
\bm{\widetilde\beta}
=& (\widetilde{\mathcal{X}}^\top\widetilde{\mathcal{X}})^{-1}\widetilde{\mathcal{X}}^\top\widetilde{\mathcal{Y}}\\
\label{beta_hat}
=&\left(1-(\widetilde{\mathcal{X}}^\top\widetilde{\mathcal{X}})^{-1}\widetilde{\mathcal{X}}^\top\bm\epsilon_X\right)\bm\beta+(\widetilde{\mathcal{X}}^\top\widetilde{\mathcal{X}})^{-1}\widetilde{\mathcal{X}}^\top(\bm\epsilon_Y+\bm\epsilon).
\nonumber
\end{align}
With a careful error analysis in Appendix \ref{app:proof of prop}, we show that $\bm{\widetilde\beta}$ is consistent if $m_D$ is consistently approximated (meaning the MSE $\frac{\bm\epsilon_X^\top\bm\epsilon_X}{n}\to\sigma_X^2=0$). Otherwise, we give a correction based on $\bm{\widetilde\beta}$.

\begin{theorem}
Under the assumption of additive measurement errors and if
\begin{align*}
\mathbb{E}(\bm\epsilon_X)&=\mathbb{E}(\bm\epsilon_Y)=0
\\
\textup{Var}(\bm\epsilon_X)&=\sigma_X^2\mathbf{I}, \quad
\textup{Var}(\bm\epsilon_Y)=\sigma_Y^2\mathbf{I},
\end{align*}
then $(\mathbf{I}-\mathbf{R})^{-1}\widetilde{\bm\beta}$ is $\sqrt{n}$-consistent and so is $\bm{\widetilde\beta}$ if and only if $\sigma_X^2=0$, with $\mathbf{R}=\sigma_X^2\left(\textup{plim}\frac{\mathcal{X}^\top\mathcal{X}}{n}+\sigma_X^2\mathbf{I}\right)^{-1}$.

Furthermore, suppose the errors are Gaussian: $\bm\epsilon\sim\mathcal{N}(0,\sigma_\epsilon^2\mathbf{I}),\bm\epsilon_X\sim\mathcal{N}(0,\sigma_X^2\mathbf{I}), \bm\epsilon_Y\sim\mathcal{N}(0,\sigma_Y^2\mathbf{I})$, then
we have the asymptotic normality
\begin{align*}
\sqrt{n}\left(\bm{\widetilde\beta}-(\mathbf{I}-\mathbf{R})\bm\beta\right)\overset{\mathcal{D}}{\to}\mathcal{N}\left(0,\frac{\sigma_\epsilon^2+\sigma_Y^2}{\sigma_X^2}\mathbf{R}\right).
\end{align*}
\label{prop:consistency}
\end{theorem}

\vspace{-0.25cm}
\section{DebiNet: Debiasing Neural Network }
\vspace{-0.15cm}
Before introducing our debiasing network, DebiNet, we revisit two broadly used approaches to debias Lasso. The first method is OLS post-Lasso \cite{belloni2013least}, which uses the features selected by the Lasso, i.e. $\vect{D}=\vect{X}_S$, to fit an OLS with $\vect{y}$ as the response variable and obtains $\bm{\hat\beta}_\text{OLS}$. The final estimator substitutes the non-zero entries in the Lasso estimator $\bm{\hat\theta}_\text{Lasso}$ by the OLS estimator:
\vspace{-0.1cm}
\begin{align}
[\bm{\hat\theta}_\text{OLS post-Lasso}]_j=
\begin{cases}
0& \text{ if } [\bm{\hat\theta}_\text{Lasso}]_j=0
\\
[\bm{\hat\beta}_\text{OLS}]_j& \text{ if } [\bm{\hat\theta}_\text{Lasso}]_j\neq 0
\end{cases}
\end{align}
The second approach is the debiased (or desparsified) Lasso \cite{zhang2014confidence,van2014asymptotically},
\vspace{-0.1cm}
\begin{align}
	\bm{\hat\theta_\text{debiased Lasso}}=\bm{\hat\theta}_\text{Lasso}+\hat\Sigma_\text{Lasso}^{-1} \vect{X}^\top(\vect{y}-\vect{X}\bm{\hat\theta}_\text{Lasso})
\end{align}
where $\hat\Sigma_\text{Lasso}^{-1}$ is a pseudo-inverse of $\vect{X}^\top \vect{X}$, obtained by nodewise Lasso regressions.

Notice that OLS post-Lasso only partly debiases $\bm{\hat\theta}_\text{Lasso}$ while the debiased Lasso debiases all elements in $\bm{\hat\theta}_\text{Lasso}$. The main difference between these approaches is the information they exploit: OLS post-Lasso requires only the information of $\vect{X}_S$ and the indicator $\mathbb{I}(\bm{\hat\theta}_\text{Lasso}\neq 0)$, and the debiased Lasso uses the entire $\vect{X}$ and $\bm{\hat\theta}_\text{Lasso}$. Nevertheless, both methods actually connect to PLM: one can view the OLS post-Lasso as a special case of PLM with $f=0$ and the nodewise Lasso regression indeed learns $m_D$ in a semi-parametric regime.

Now we are ready to present our debiasing model.
\begin{algorithm}[!htb]
\centering 
\caption{DebiNet}
\begin{algorithmic}
\STATE{\textbf{Input}: Data matrix $\mat{X}$, label $\vect{y}$} \\
\STATE{\quad\quad 1. fit $\vect y\sim \mat{X}$ via any feature selection model to obtain $\mat{D}:=\mat{X}_{S}$ and $\mat{Z}:=\mat{X}_{S}^C$;}
\\
\STATE{\textbf{Estimation of $\bm\beta$:}}
\\
\STATE{\quad\quad 2. fit $[\vect y,\mat D]\sim \mat Z$ via over-parameterized neural network to derive $\mathbb{E}([\vect y,\mat D]|\mat Z)$;}
\\
\STATE{\quad\quad 3. fit $\vect y-\mathbb{E}(\vect y|\mat Z)\sim \mat D-\mathbb{E}(\mat D|\mat Z)$ via OLS to derive $\bm{\hat\beta}$;}
\\
\STATE{\textbf{Prediction of $\vect y$ and Estimation of $f$:}}
\\
\STATE{\quad\quad 4. define $\vect {\hat y}:=\mathbb{E}(\vect y|\mat Z)+(\mat D-\mathbb{E}(\mat D|\mat Z))\bm{\hat\beta}$ and $\hat f(\mat Z):=\mathbb{E}(\vect y|\mat Z)-\mathbb{E}(\mat D|\mat Z)\bm{\hat\beta}$.}
\end{algorithmic}
\end{algorithm}

Similar to the debiased Lasso, DebiNet also exploits the unselected features and thus avoids the model mis-specification error in OLS post-Lasso. For example, suppose that the sparsity of $\bm\beta$ is $k$ out of $p$ (i.e. there are $k$ important features) and that Lasso fails to select all them, then it is impossible for the OLS post-Lasso model ($\mathbf{\hat y}_\text{OLS post-Lasso}=\mathbf{D}\bm{\hat\beta}_\text{OLS}$) to learn the true model nor to be consistent. However, under the debiased Lasso ($\mathbf{\hat y}_\text{debiased Lasso}=\mathbf{X}\bm{\hat\theta}_\text{debiased Lasso}$) and DebiNet, the true model is learnable. We highlight that the mis-specification (when true positive rate is not 1) is common in high dimension and high sparsity \cite{su2017false,donoho2009observed,donoho2005neighborly,donoho2006high}, known as the Donoho-Tanner phase transition. Therefore, it is vital to remedy the failure of variable selection by using the signals left in $\mathbf{Z}$ (see Table \ref{table2:debiasing}).

\begin{table*}[!htb]
\centering
\begin{tabular}{|cccccc|}
\hline 
Methods & Estimation MSE & Train MSE & Test MSE & 95\% Coverage &  sec/run\\
\hline 
\toprule
\multicolumn{6}{|c|}{$n=1000$; High dimension $p=3000$; High sparsity $k=300$}\\
Lasso & 0.536($\pm$0.058 ) & 286.218($\pm$9.766 ) & 296.112($\pm$25.615 ) & -& -\\
OLS post-Lasso & 1.402($\pm$0.309 ) & 233.363($\pm$11.050 ) & 308.440($\pm$25.241 ) & 0.558& -\\
debiased Lasso & 2.019($\pm$0.290 ) & 4998.75($\pm$364.669 ) & 1100.68($\pm$111.53 ) & 0.000 & 154\\
DebiNet & \textbf{0.421}($\pm$0.418 ) & \textbf{136.376}($\pm$54.195 ) & \textbf{278.625}($\pm$30.932 ) & \textbf{0.830}& 7(60)\\
NW post-Lasso & 1.398 ($\pm$0.307 ) & 229.193 ($\pm$10.860 ) & 308.399 ($\pm$25.301 ) & 0.560& -\\
\midrule
\multicolumn{6}{|c|}{$n=1000$; High dimension $p=3000$; Low sparsity $k=10$}\\
Lasso & 0.264($\pm$0.027 ) & 3.531 ($\pm$0.131 ) & 3.634 ($\pm$0.411 ) & -& -\\
OLS post-Lasso & \textbf{0.001} ($\pm$ 0.000 ) & \textbf{0.992} ($\pm$0.057 ) & \textbf{1.009} ($\pm$0.086 ) & \textbf{0.958}& -\\
debiased Lasso & 0.004 ($\pm$0.002 ) & 55.024 ($\pm$3.295 ) & 14.000 ($\pm$1.410 ) & 0.702 & 150\\
DebiNet & \textbf{0.001} ($\pm$0.000 ) & \textbf{0.989} ($\pm$0.056 ) & \textbf{1.012} ($\pm$0.086 ) & \textbf{0.956}& 3(21)\\
NW post-Lasso & \textbf{0.001} ($\pm$0.000 ) & \textbf{0.975} ($\pm$0.056 ) & \textbf{1.009} ($\pm$0.084 ) & \textbf{0.960}& -\\
\midrule
\multicolumn{6}{|c|}{$n=1000$; Low dimension $p=500$; High sparsity $k=400$}\\
OLS & \textbf{0.003} ($\pm$0.000 ) & \textbf{0.378} ($\pm$0.031 ) & \textbf{2.625} ($\pm$0.313 ) & \textbf{0.955}& -\\
Lasso & 0.458 ($\pm$0.021 ) & 175.803 ($\pm$3.846 ) & 245.660 ($\pm$25.020 ) & -& -\\
OLS post-Lasso & 0.175 ($\pm$0.025 ) & 88.959 ($\pm$6.314 ) & 211.952 ($\pm$21.839 ) & 0.927& -\\
debiased Lasso & 0.299 ($\pm$0.036 ) & 122.515 ($\pm$12.222 ) & 122.204 ($\pm$13.268 ) & 0.123& 4\\
DebiNet & 0.175 ($\pm$0.025 ) & 88.953 ($\pm$6.324 ) & 211.877 ($\pm$21.926 ) & 0.928& 2(6)\\
NW post-Lasso & 0.174 ($\pm$0.025 ) & 87.132 ($\pm$6.195 ) & 211.007 ($\pm$21.821 ) & 0.927& -
\\
\bottomrule
\end{tabular}
\caption{Comparison of debiasing methods in 50 independent runs. See Appendix \ref{table2:data} for data generation details. In the `sec/run' column, `-' means $<1$ second. Here we record two time for DebiNet, one for GPU acceleration and the other in bracket for CPU. NW post-Lasso applies PLM-NW, instead of PLM-NN, after Lasso.}
\label{table2:debiasing}
\end{table*}

\begin{figure*}[!htb]
\centering
\hbox{
\begin{subfigure}{0.48\linewidth}\includegraphics[width=\linewidth]{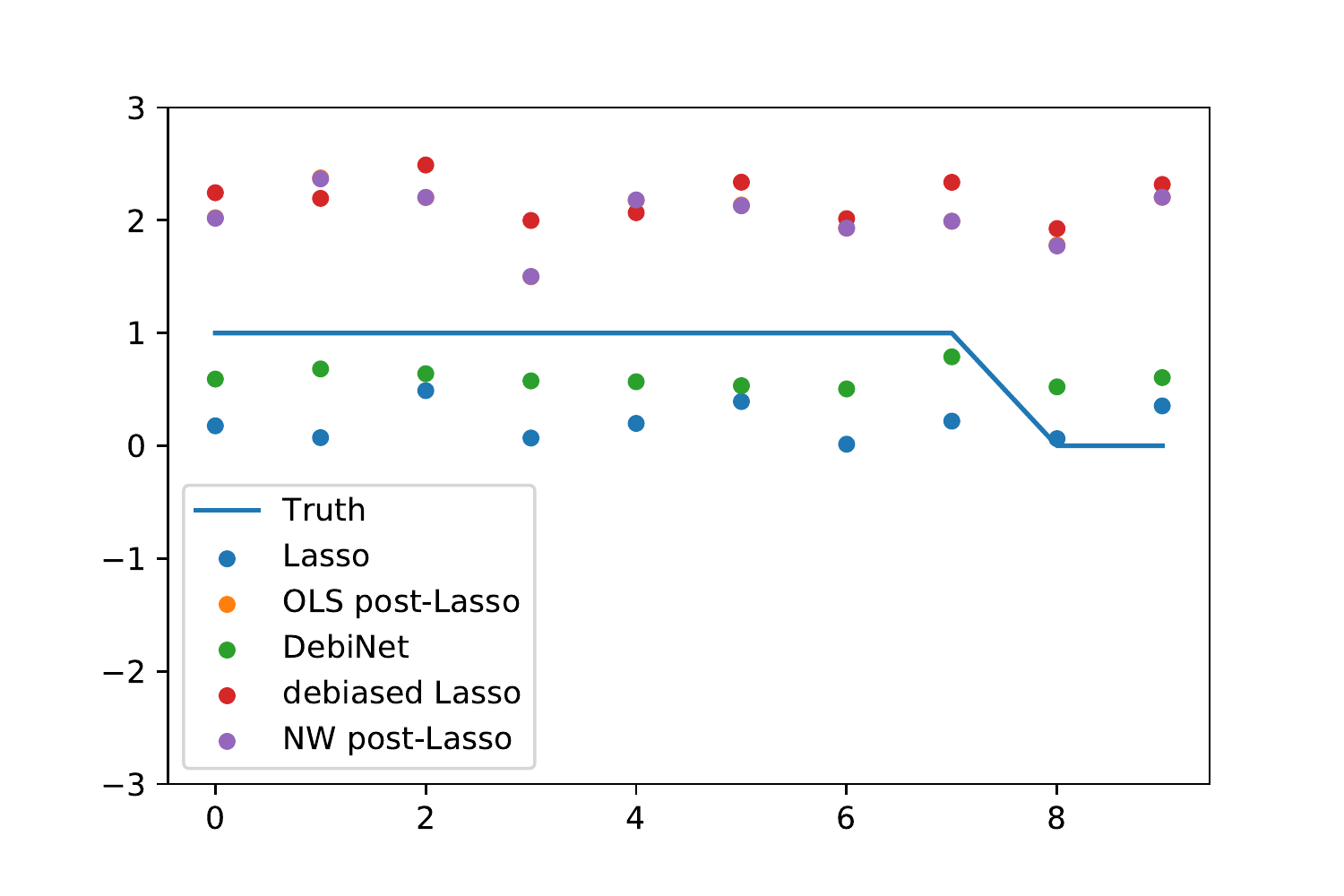}
\end{subfigure}
\hspace{-1cm}
\begin{subfigure}{0.48\linewidth}\includegraphics[width=\linewidth]{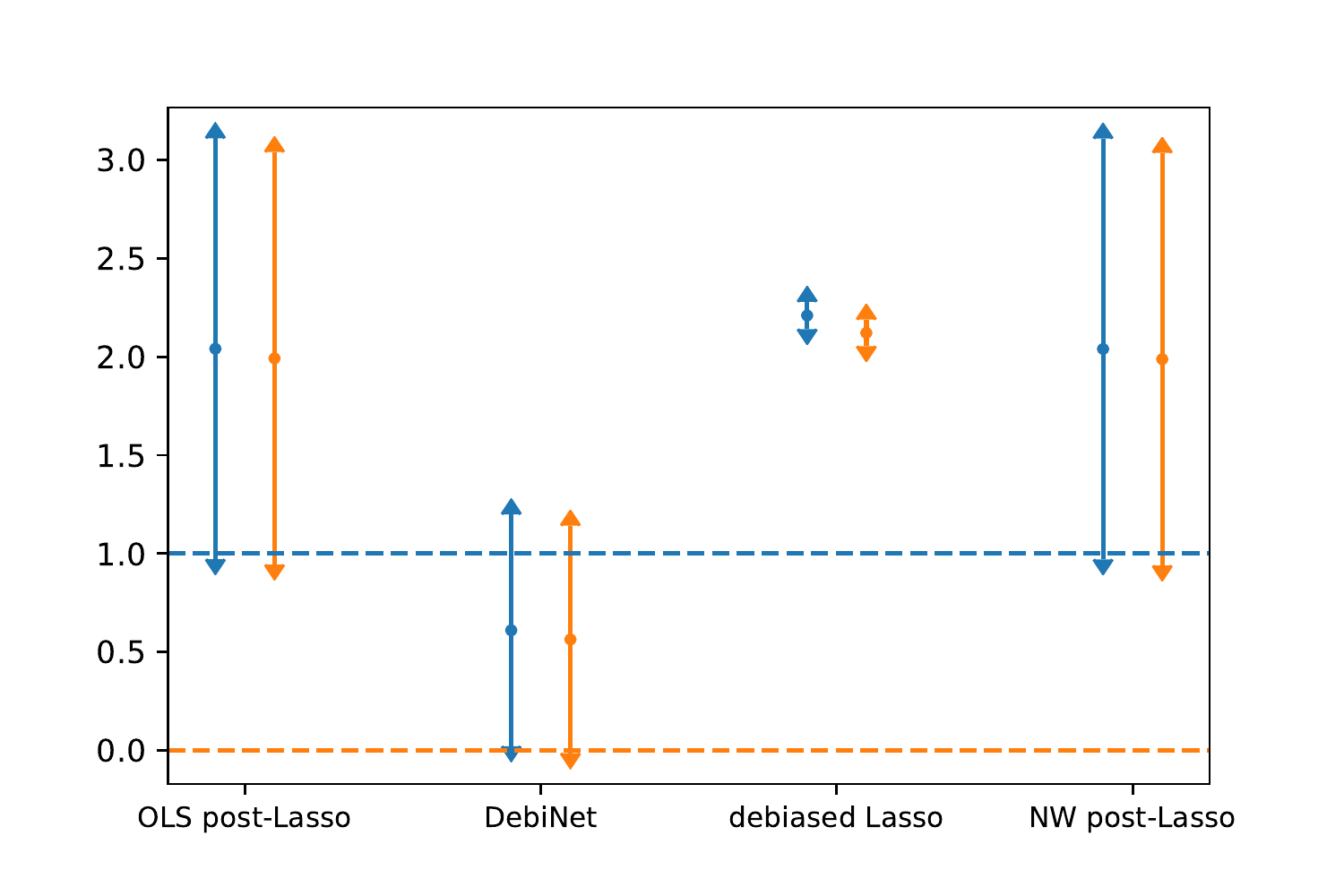}
\end{subfigure}
}
\vspace{-0.3cm}
\hbox{
\begin{subfigure}{0.48\linewidth}\includegraphics[width=\linewidth]{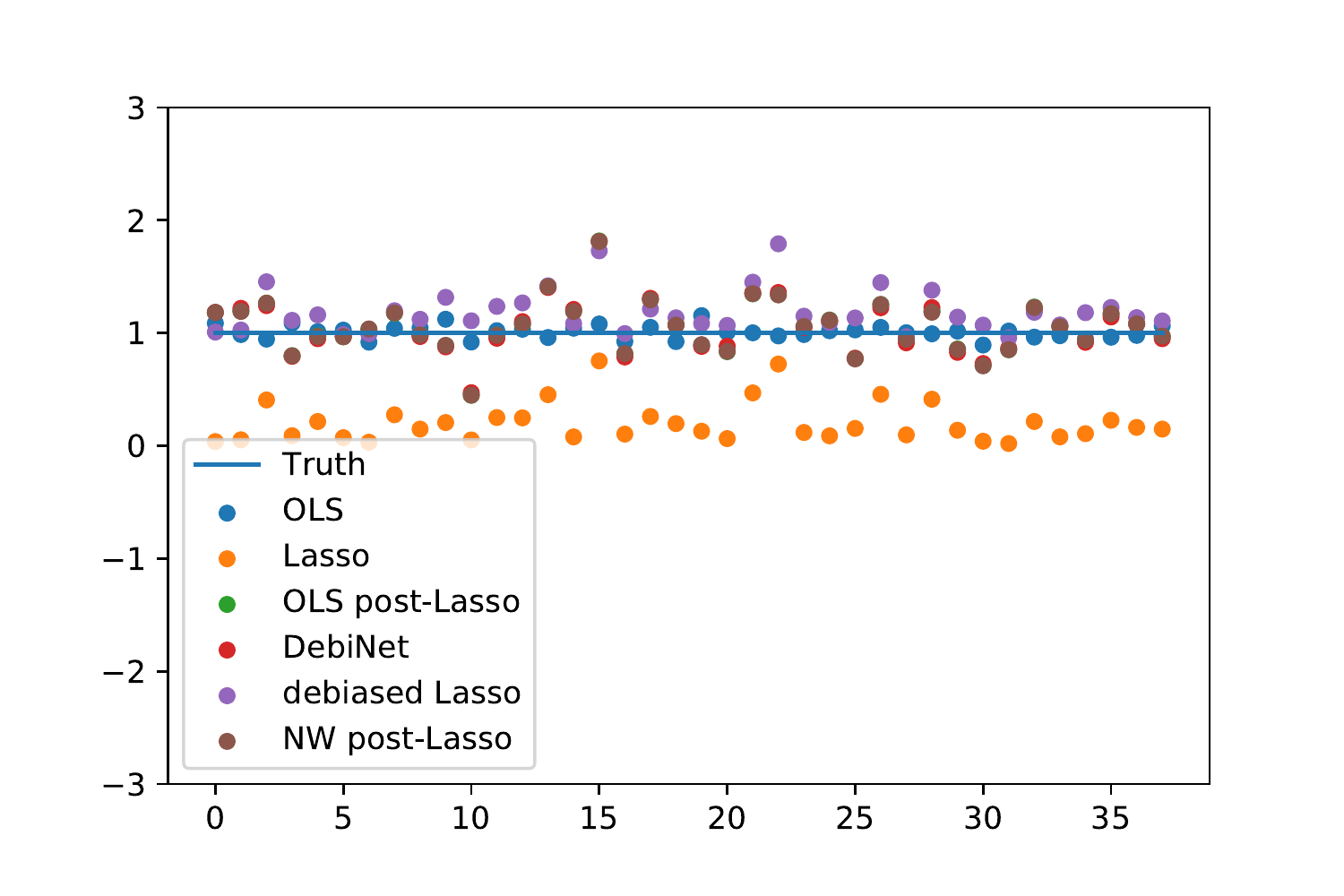}
\end{subfigure}
\hspace{-1cm}
\begin{subfigure}{0.48\linewidth}\includegraphics[width=\linewidth]{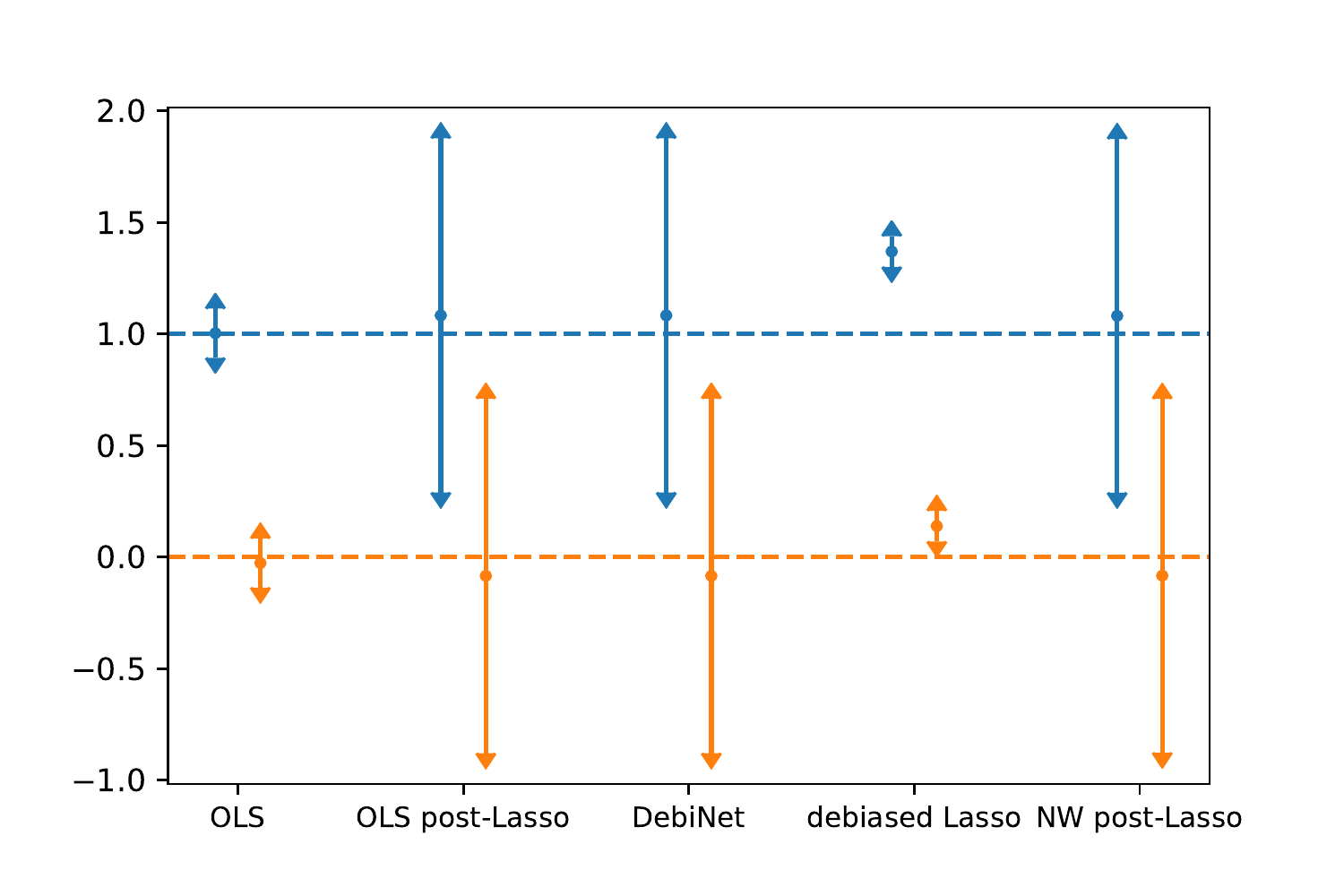}
\end{subfigure}
}
\vspace{-0.3cm}
\caption{Comparison of debiasing methods on Lasso. $\mathbf X\in\mathbb{R}^{1000\times p}, \text{Var}(\epsilon)=1$. Top plots: $p= 3000,\lambda=2, k=300$. Bottom plots: $p=500,\lambda=1$; left $k=100$; right $k=300$. $\bm\beta$ is binary and we plot the $95\%$ confidence intervals according to $\beta_j=1$ (blue) and $\beta_j=0$ (orange).}
\label{fig:low dim}
\end{figure*}

On the other hand, DebiNet also shares some similarity with the OLS post-Lasso as both methods only need the information in $\mathbb{I}(\bm{\hat\theta}_\text{Lasso}\neq 0)$, both debias the parameters of interest $\bm\beta$ and both make use of the OLS which requires $p_L<n$. It is remarkable that DebiNet only fits three models and the OLS post-Lasso fits two models, while the debiased Lasso needs $p$ nodewise regression to construct $\hat\Sigma_\text{Lasso}^{-1}$. The computation cost can be immense in high dimension.

In Table \ref{table2:debiasing} and Figure \ref{fig:low dim}, we compare different debiasing methods under various metrics over the active set $\{j:[\bm{\hat\theta}_{\text{Lasso}}]_j\neq 0\}$. 
We emphasize that, since this work focuses on the consistency of estimation and the effect of debiasing, we must experiment on synthetic data where the truth $\bm\beta$ is known, in order to compute the estimation MSE and the statistical coverge. In all settings, DebiNet outperforms other methods (except the golden rule, OLS, in the low dimension) and demonstrates its robustness against high dimension and high sparsity.
We note that the debiased Lasso exhibits strange behaviors in very high dimension as its working assumption requires the sparsity $k=o(\sqrt{n}/\log{p})$ which is roughly 9. In our high sparsity setting, $k=300\gg 9$ and the debiased Lasso, giving narrow confidence intervals, may not perform well. In addition, the objective of the debiased Lasso is to reduce the estimation error, instead of to minimize the prediction error. These two goals are unified in low dimension but not necessarily in high dimension. 
Moreover, the coverage is computed over the active set of Lasso, while the debiased Lasso may offer higher coverage outside the active set \cite{li2017debiasing}. Lastly, we remark that arbitrary variable selection methods can be plugged into DebiNet, as long as the true model is linear and the number of selected variables is less than the sample size: the feature selection only affects the choice of $[\mathbf{D},\mathbf{Z}]$, but not the subsequent convergence and consistency analyses.
\vspace{-0.25cm}
\section{A Study of Treatment Effect on 401(k) Data}
\label{sec:7}
\vspace{-0.15cm}

In this section, we analyze the 1991 Survey of Income and Program Participation data in \cite{chernozhukov2018double}, to study the average treatment effect (ATE) which estimates the impact of household income (our treatment variable $\mathbf{D}\in\mathbb{R}^{9915\times 1}$) on the net financial assets (our response variable $\mathbf{y}\in\mathbb{R}^{9915}$), after accounting for the confounding factors $\mathbf{Z}\in\mathbb{R}^{9915\times 9}$ such as age, family size, years of
education, etc. The ATE, or $\beta\in\mathbb{R}$, is expected to be positive according to \cite{poterba1994401,poterba1995401}: when the 401(k) plan program started, people did not base job decisions on the retirement offers but rather on the income. 
\begin{figure}
\includegraphics[width=0.48\textwidth]{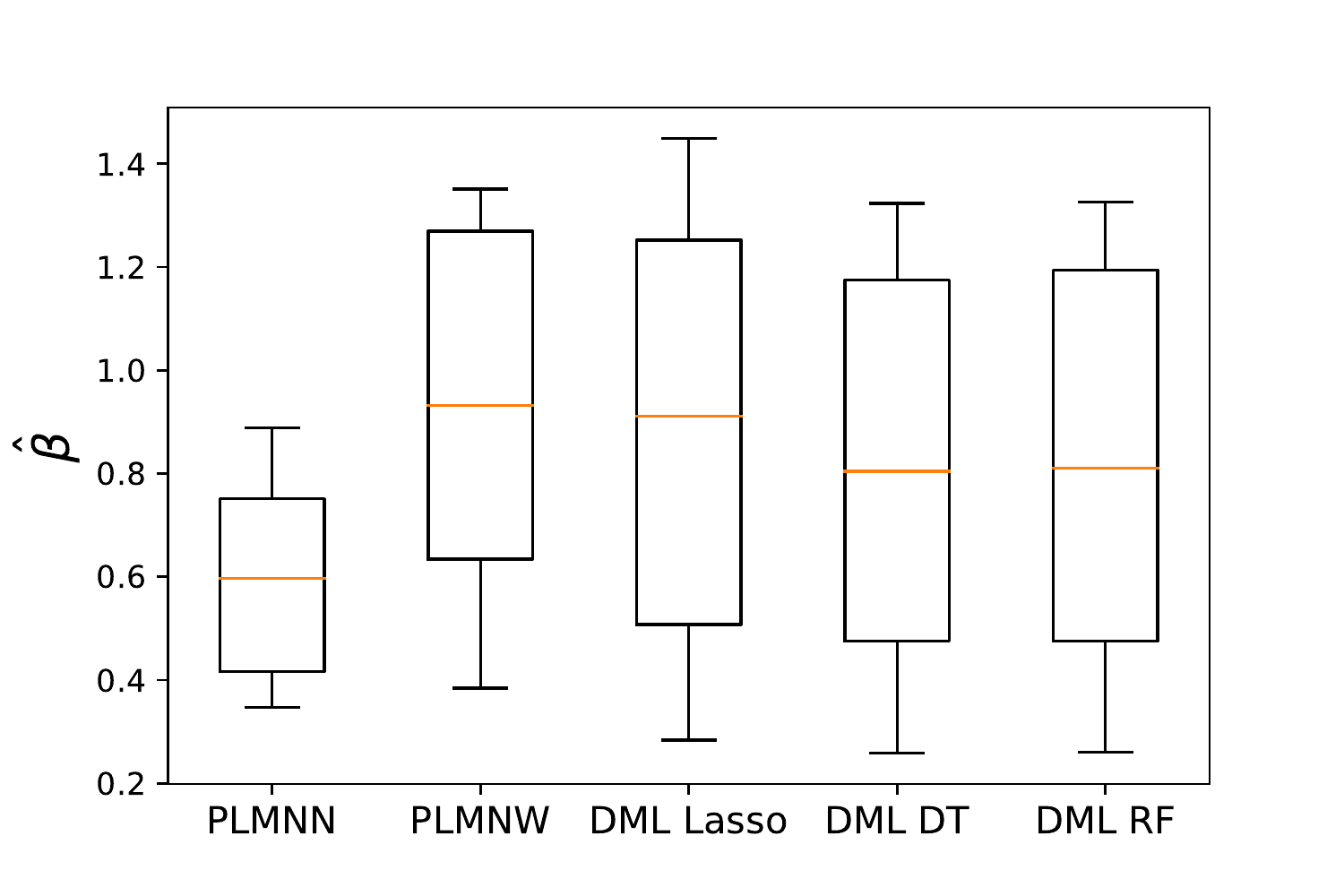}
\captionof{figure}{Estimation of $\bm{\hat\beta}$ by different PLMs.}
\label{fig:compare with PLMs}
\end{figure}
\begin{table}[]
    \centering
\begin{tabular}{|cccc|}
\hline 
Methods & Mean $\hat\beta$ & Median $\hat\beta$& SE $\hat\beta$
\\
\hline
\toprule
PLM-NN & 0.590 & 0.596 & 0.186 \\
PLM-NW & 0.912 & 0.931 & 0.351\\
DML Lasso & 0.887 & 0.911 & 0.415 \\
DML DT & 0.805 & 0.804 & 0.385\\
DML RF & 0.809 & 0.810 & 0.388\\
\bottomrule
\end{tabular}
\captionof{table}{Comparison of treatment effects by different PLMs on the 401(k) dataset.}
\label{table:compare with debias}
\end{table}

Reassuringly, the results obtained from different PLMs are relatively consistent with each other across different sub-samplings. We note that the PLM theory indeed supports such consistency and thus over-parameterized neural networks can fit in DML, even without sample-splitting. We also observe that PLM-NN has much smaller standard error (SE) than other methods, suggesting that it can be more efficient in estimation.

\vspace{-0.25cm}
\section{Discussion}
\vspace{-0.15cm}
In this paper, we propose to use over-parameterized neural networks in the partially linear model (PLM-NN). We show that PLM-NN is computationally efficient and provably converges to global minimum exponentially fast. We demonstrate its robustness to high dimension and high sparsity, and its strong performance in terms of the estimation and prediction errors. Based on this new PLM, we design DebiNet, a debiasing framework for arbitrary feature selection method and illustrate its potential to infer and to predict accurately (e.g. in adversarial attack in Appendix \ref{AppendixD2}). Now we discuss some future directions.

As also discussed in \cite{du2018gradient,du2018gradient2}, our framework can easily extend to deep over-parameterized neural networks: for $L$-hidden-layer neural network with equally wide $l$-th layer $\{\w_r^{(l)}\}_{r\in [m]}$, it has been shown that the last hidden layer's kernel (similar to the notation in \eqref{eqn:H_t}): $\sum_{r=1}^{m}\langle\frac{\partial F_s(\mat{W},\vect{A},\vect{Z}_i)}{\partial \vect{w}_r^{(L)}}, \frac{\partial F_h(\mat{W},\mat{A},\vect{Z}_j)}{\partial \vect{w}_r^{(L)}}\rangle$ is positive definite, and so is the overall NTK with respect to all weights. Hence we can guarantee linear convergence for deep over-parameterized neural networks. In addition, the width requirement of the hidden layer can be tightened by more advanced matrix perturbation theory. Based on our experiments in Appendix \ref{AppendixE}, we may also use other activation functions, optimizers (e.g. Adam \cite{kingma2014adam}) and losses (e.g. Huber loss \cite{huber1992robust}) and still observe the linear convergence. 
We note that PLM-NN does not perform well when $\vect{y}$ or $\vect{D}$ is discrete, or when $\vect{y}$ and $\vect{D}$ have different scales. It would be desirable to robustify the model against the input distributional assumption, e.g. using the generalized partially linear model within DebiNet.
\newpage

\medskip

\printbibliography

@article{du2018gradient,
  title={Gradient descent provably optimizes over-parameterized neural networks},
  author={Du, Simon S and Zhai, Xiyu and Poczos, Barnabas and Singh, Aarti},
  journal={arXiv preprint arXiv:1810.02054},
  year={2018}
}

@article{hornik1991approximation,
  title={Approximation capabilities of multilayer feedforward networks},
  author={Hornik, Kurt},
  journal={Neural networks},
  volume={4},
  number={2},
  pages={251--257},
  year={1991},
  publisher={Elsevier}
}

@inproceedings{stinchcombe1989universal,
  title={Universal approximation using feedforward networks with non-sigmoid hidden layer activation functions},
  author={Stinchcombe, Maxwell and White, Halbert},
  booktitle={IJCNN International Joint Conference on Neural Networks},
  year={1989}
}

@article{gybenko1989approximation,
  title={Approximation by superposition of sigmoidal functions},
  author={Gybenko, G},
  journal={Mathematics of Control, Signals and Systems},
  volume={2},
  number={4},
  pages={303--314},
  year={1989}
}

@article{leshno1993multilayer,
  title={Multilayer feedforward networks with a nonpolynomial activation function can approximate any function},
  author={Leshno, Moshe and Lin, Vladimir Ya and Pinkus, Allan and Schocken, Shimon},
  journal={Neural networks},
  volume={6},
  number={6},
  pages={861--867},
  year={1993},
  publisher={Elsevier}
}

@article{white1990connectionist,
  title={Connectionist nonparametric regression: Multilayer feedforward networks can learn arbitrary mappings},
  author={White, Halbert},
  journal={Neural networks},
  volume={3},
  number={5},
  pages={535--549},
  year={1990},
  publisher={Elsevier}
}

@article{arora2019fine,
  title={Fine-grained analysis of optimization and generalization for overparameterized two-layer neural networks},
  author={Arora, Sanjeev and Du, Simon S and Hu, Wei and Li, Zhiyuan and Wang, Ruosong},
  journal={arXiv preprint arXiv:1901.08584},
  year={2019}
}

@inproceedings{allen2019learning,
  title={Learning and generalization in overparameterized neural networks, going beyond two layers},
  author={Allen-Zhu, Zeyuan and Li, Yuanzhi and Liang, Yingyu},
  booktitle={Advances in neural information processing systems},
  pages={6155--6166},
  year={2019}
}

@article{engle1986semiparametric,
  title={Semiparametric estimates of the relation between weather and electricity sales},
  author={Engle, Robert F and Granger, Clive WJ and Rice, John and Weiss, Andrew},
  journal={Journal of the American statistical Association},
  volume={81},
  number={394},
  pages={310--320},
  year={1986},
  publisher={Taylor \& Francis Group}
}

@book{hardle2012partially,
  title={Partially linear models},
  author={H{\"a}rdle, Wolfgang and Liang, Hua and Gao, Jiti},
  year={2012},
  publisher={Springer Science \& Business Media}
}

@article{robinson1988root,
  title={Root-N-consistent semiparametric regression},
  author={Robinson, Peter M},
  journal={Econometrica: Journal of the Econometric Society},
  pages={931--954},
  year={1988},
  publisher={JSTOR}
}

@article{hamilton1997local,
  title={Local linear estimation in partly linear models},
  author={Hamilton, Scott A and Truong, Young K},
  journal={Journal of Multivariate Analysis},
  volume={60},
  number={1},
  pages={1--19},
  year={1997},
  publisher={Elsevier}
}

@misc{chernozhukov2018double,
  title={Double/debiased machine learning for treatment and structural parameters},
  author={Chernozhukov, Victor and Chetverikov, Denis and Demirer, Mert and Duflo, Esther and Hansen, Christian and Newey, Whitney and Robins, James},
  year={2018},
  publisher={Oxford University Press Oxford, UK}
}

@article{bach2017breaking,
  title={Breaking the curse of dimensionality with convex neural networks},
  author={Bach, Francis},
  journal={The Journal of Machine Learning Research},
  volume={18},
  number={1},
  pages={629--681},
  year={2017},
  publisher={JMLR. org}
}

@incollection{poterba1994401,
  title={401 (k) plans and tax-deferred saving},
  author={Poterba, James M and Venti, Steven F},
  booktitle={Studies in the Economics of Aging},
  pages={105--142},
  year={1994},
  publisher={University of Chicago Press}
}

@article{poterba1995401,
  title={Do 401 (k) contributions crowd out other personal saving?},
  author={Poterba, James M and Venti, Steven F and Wise, David A},
  journal={Journal of Public Economics},
  volume={58},
  number={1},
  pages={1--32},
  year={1995},
  publisher={Elsevier}
}

@article{belloni2013least,
  title={Least squares after model selection in high-dimensional sparse models},
  author={Belloni, Alexandre and Chernozhukov, Victor and others},
  journal={Bernoulli},
  volume={19},
  number={2},
  pages={521--547},
  year={2013},
  publisher={Bernoulli Society for Mathematical Statistics and Probability}
}

@article{li2017debiasing,
  title={Debiasing the Debiased Lasso with Bootstrap},
  author={Li, Sai},
  journal={arXiv preprint arXiv:1711.03613},
  year={2017}
}

@article{zhang2014confidence,
  title={Confidence intervals for low dimensional parameters in high dimensional linear models},
  author={Zhang, Cun-Hui and Zhang, Stephanie S},
  journal={Journal of the Royal Statistical Society: Series B (Statistical Methodology)},
  volume={76},
  number={1},
  pages={217--242},
  year={2014},
  publisher={Wiley Online Library}
}

@article{van2014asymptotically,
  title={On asymptotically optimal confidence regions and tests for high-dimensional models},
  author={Van de Geer, Sara and B{\"u}hlmann, Peter and Ritov, Ya’acov and Dezeure, Ruben and others},
  journal={The Annals of Statistics},
  volume={42},
  number={3},
  pages={1166--1202},
  year={2014},
  publisher={Institute of Mathematical Statistics}
}

@article{su2017false,
  title={False discoveries occur early on the lasso path},
  author={Su, Weijie and Bogdan, Ma{\l}gorzata and Candes, Emmanuel and others},
  journal={The Annals of statistics},
  volume={45},
  number={5},
  pages={2133--2150},
  year={2017},
  publisher={Institute of Mathematical Statistics}
}

@article{taylor2015statistical,
  title={Statistical learning and selective inference},
  author={Taylor, Jonathan and Tibshirani, Robert J},
  journal={Proceedings of the National Academy of Sciences},
  volume={112},
  number={25},
  pages={7629--7634},
  year={2015},
  publisher={National Acad Sciences}
}

@article{berk2013valid,
  title={Valid post-selection inference},
  author={Berk, Richard and Brown, Lawrence and Buja, Andreas and Zhang, Kai and Zhao, Linda and others},
  journal={The Annals of Statistics},
  volume={41},
  number={2},
  pages={802--837},
  year={2013},
  publisher={Institute of Mathematical Statistics}
}

@article{lee2016exact,
  title={Exact post-selection inference, with application to the lasso},
  author={Lee, Jason D and Sun, Dennis L and Sun, Yuekai and Taylor, Jonathan E and others},
  journal={The Annals of Statistics},
  volume={44},
  number={3},
  pages={907--927},
  year={2016},
  publisher={Institute of Mathematical Statistics}
}

@article{tibshirani2016exact,
  title={Exact post-selection inference for sequential regression procedures},
  author={Tibshirani, Ryan J and Taylor, Jonathan and Lockhart, Richard and Tibshirani, Robert},
  journal={Journal of the American Statistical Association},
  volume={111},
  number={514},
  pages={600--620},
  year={2016},
  publisher={Taylor \& Francis}
}

@article{liang1999estimation,
  title={Estimation in a semiparametric partially linear errors-in-variables model},
  author={Liang, Hua and H{\"a}rdle, Wolfgang and Carroll, Raymond J and others},
  journal={The Annals of Statistics},
  volume={27},
  number={5},
  pages={1519--1535},
  year={1999},
  publisher={Institute of Mathematical Statistics}
}

@article{heckman1986spline,
  title={Spline smoothing in a partly linear model},
  author={Heckman, Nancy E},
  journal={Journal of the Royal Statistical Society: Series B (Methodological)},
  volume={48},
  number={2},
  pages={244--248},
  year={1986},
  publisher={Wiley Online Library}
}

@article{chen1988convergence,
  title={Convergence rates for parametric components in a partly linear model},
  author={Chen, Hung and others},
  journal={The Annals of Statistics},
  volume={16},
  number={1},
  pages={136--146},
  year={1988},
  publisher={Institute of Mathematical Statistics}
}

@article{speckman1988kernel,
  title={Kernel smoothing in partial linear models},
  author={Speckman, Paul},
  journal={Journal of the Royal Statistical Society: Series B (Methodological)},
  volume={50},
  number={3},
  pages={413--436},
  year={1988},
  publisher={Wiley Online Library}
}

@inproceedings{jacot2018neural,
  title={Neural tangent kernel: Convergence and generalization in neural networks},
  author={Jacot, Arthur and Gabriel, Franck and Hongler, Cl{\'e}ment},
  booktitle={Advances in neural information processing systems},
  pages={8571--8580},
  year={2018}
}

@article{efroymson1960multiple,
  title={Multiple regression analysis},
  author={Efroymson, MA},
  journal={Mathematical methods for digital computers},
  pages={191--203},
  year={1960},
  publisher={John Wiley \& Sons}
}

@article{tibshirani1996regression,
  title={Regression shrinkage and selection via the lasso},
  author={Tibshirani, Robert},
  journal={Journal of the Royal Statistical Society: Series B (Methodological)},
  volume={58},
  number={1},
  pages={267--288},
  year={1996},
  publisher={Wiley Online Library}
}

@article{zou2006adaptive,
  title={The adaptive lasso and its oracle properties},
  author={Zou, Hui},
  journal={Journal of the American statistical association},
  volume={101},
  number={476},
  pages={1418--1429},
  year={2006},
  publisher={Taylor \& Francis}
}

@article{zou2005regularization,
  title={Regularization and variable selection via the elastic net},
  author={Zou, Hui and Hastie, Trevor},
  journal={Journal of the royal statistical society: series B (statistical methodology)},
  volume={67},
  number={2},
  pages={301--320},
  year={2005},
  publisher={Wiley Online Library}
}

@article{friedman2010note,
  title={A note on the group lasso and a sparse group lasso},
  author={Friedman, Jerome and Hastie, Trevor and Tibshirani, Robert},
  journal={arXiv preprint arXiv:1001.0736},
  year={2010}
}

@article{simon2013sparse,
  title={A sparse-group lasso},
  author={Simon, Noah and Friedman, Jerome and Hastie, Trevor and Tibshirani, Robert},
  journal={Journal of computational and graphical statistics},
  volume={22},
  number={2},
  pages={231--245},
  year={2013},
  publisher={Taylor \& Francis Group}
}

@article{bogdan2015slope,
  title={SLOPE—adaptive variable selection via convex optimization},
  author={Bogdan, Ma{\l}gorzata and Van Den Berg, Ewout and Sabatti, Chiara and Su, Weijie and Cand{\`e}s, Emmanuel J},
  journal={The annals of applied statistics},
  volume={9},
  number={3},
  pages={1103},
  year={2015},
  publisher={NIH Public Access}
}

@article{park2008bayesian,
  title={The bayesian lasso},
  author={Park, Trevor and Casella, George},
  journal={Journal of the American Statistical Association},
  volume={103},
  number={482},
  pages={681--686},
  year={2008},
  publisher={Taylor \& Francis}
}

@article{hans2009bayesian,
  title={Bayesian lasso regression},
  author={Hans, Chris},
  journal={Biometrika},
  volume={96},
  number={4},
  pages={835--845},
  year={2009},
  publisher={Oxford University Press}
}

@article{chernozhukov2015post,
  title={Post-selection and post-regularization inference in linear models with many controls and instruments},
  author={Chernozhukov, Victor and Hansen, Christian and Spindler, Martin},
  journal={American Economic Review},
  volume={105},
  number={5},
  pages={486--90},
  year={2015}
}

@article{taylor2014post,
  title={Post-selection adaptive inference for least angle regression and the lasso},
  author={Taylor, Jonathan and Lockhart, Richard and Tibshirani, Ryan J and Tibshirani, Robert},
  journal={arXiv preprint arXiv:1401.3889},
  volume={354},
  year={2014}
}

@inproceedings{lee2019wide,
  title={Wide neural networks of any depth evolve as linear models under gradient descent},
  author={Lee, Jaehoon and Xiao, Lechao and Schoenholz, Samuel and Bahri, Yasaman and Novak, Roman and Sohl-Dickstein, Jascha and Pennington, Jeffrey},
  booktitle={Advances in neural information processing systems},
  pages={8570--8581},
  year={2019}
}

@article{du2018gradient2,
  title={Gradient descent finds global minima of deep neural networks},
  author={Du, Simon S and Lee, Jason D and Li, Haochuan and Wang, Liwei and Zhai, Xiyu},
  journal={arXiv preprint arXiv:1811.03804},
  year={2018}
}

@inproceedings{cao2019generalization,
  title={Generalization bounds of stochastic gradient descent for wide and deep neural networks},
  author={Cao, Yuan and Gu, Quanquan},
  booktitle={Advances in Neural Information Processing Systems},
  pages={10835--10845},
  year={2019}
}

@article{li1996root,
  title={On the root-n-consistent semiparametric estimation of partially linear models},
  author={Li, Qi},
  journal={Economics Letters},
  volume={51},
  number={3},
  pages={277--285},
  year={1996},
  publisher={Elsevier}
}

@article{donoho2009observed,
  title={Observed universality of phase transitions in high-dimensional geometry, with implications for modern data analysis and signal processing},
  author={Donoho, David and Tanner, Jared},
  journal={Philosophical Transactions of the Royal Society A: Mathematical, Physical and Engineering Sciences},
  volume={367},
  number={1906},
  pages={4273--4293},
  year={2009},
  publisher={The Royal Society Publishing}
}

@article{donoho2006high,
  title={High-dimensional centrally symmetric polytopes with neighborliness proportional to dimension},
  author={Donoho, David L},
  journal={Discrete \& Computational Geometry},
  volume={35},
  number={4},
  pages={617--652},
  year={2006},
  publisher={Springer}
}

@article{donoho2005neighborly,
  title={Neighborly polytopes and sparse solutions of underdetermined linear equations},
  author={Donoho, David L},
  year={2005},
  publisher={Citeseer}
}

@incollection{huber1992robust,
  title={Robust estimation of a location parameter},
  author={Huber, Peter J},
  booktitle={Breakthroughs in statistics},
  pages={492--518},
  year={1992},
  publisher={Springer}
}

@article{kingma2014adam,
  title={Adam: A method for stochastic optimization},
  author={Kingma, Diederik P and Ba, Jimmy},
  journal={arXiv preprint arXiv:1412.6980},
  year={2014}
}

@inproceedings{neyshabur2019role,
  title={The role of over-parametrization in generalization of neural networks},
  author={Neyshabur, Behnam and Li, Zhiyuan and Bhojanapalli, Srinadh and LeCun, Yann and Srebro, Nathan},
  booktitle={7th International Conference on Learning Representations, ICLR 2019},
  year={2019}
}

@article{novak2018sensitivity,
  title={Sensitivity and generalization in neural networks: an empirical study},
  author={Novak, Roman and Bahri, Yasaman and Abolafia, Daniel A and Pennington, Jeffrey and Sohl-Dickstein, Jascha},
  journal={arXiv preprint arXiv:1802.08760},
  year={2018}
}

@article{park2019effect,
  title={The effect of network width on stochastic gradient descent and generalization: an empirical study},
  author={Park, Daniel S and Sohl-Dickstein, Jascha and Le, Quoc V and Smith, Samuel L},
  journal={arXiv preprint arXiv:1905.03776},
  year={2019}
}

\clearpage
\onecolumn

\aistatstitle{Supplement Material for `DebiNet: Debiasing Linear Models with Nonlinear Overparameterized Neural Networks'}
\appendix
\section{Preliminaries of Partially Linear Models}
In this section, we revisit some partially linear models from an algorithmic perspective. The first one is the PLM with kernel regressions.

\begin{algorithm}[!htb]
	\centering 
	\caption{Partially Linear Model with Kernels (PLM-NW)}
	\begin{algorithmic}
	\STATE{\textbf{Input}: Data matrix $[\mat{D},\mat{Z}]$, label $\vect{y}$} \\
	\STATE{\textbf{Estimation of $\bm\beta$:}}
	\\
	\STATE{\quad\quad 1. fit $\vect{y}\sim \mat{Z}$ via kernel regression to derive $\mathbb{E}(\vect{y}|\mat{Z})$;}
	\\
	\STATE{\quad\quad 2. fit $\mat{D}\sim \mat{Z}$ via kernel regression to derive $\mathbb{E}(\mat{D}|\mat{Z})$;}
	\\
	\STATE{\quad\quad 3. fit $\vect{y}-\mathbb{E}(\vect{y}|\mat{Z})\sim \mat{D}-\mathbb{E}(\mat{D}|\mat{Z})$ via OLS to derive $\bm{\hat{\beta}}$;}
	\\
	\STATE{\textbf{Estimation of $f$ and Prediction of $\vect{y}$:}}
	\\
	\STATE{\quad\quad 4. fit $\vect{y}-\mat{D}\bm{\hat{\beta}}\sim \mat{Z}$ via kernel regression to derive $\hat f$ and define $\vect{\hat {y}}=\mat{D}\bm{\hat{\beta}}+\hat f(\mat{Z})$.}
	\end{algorithmic}
\label{alg:plm kernel}
\end{algorithm}
The kernel regression estimates the conditional expectation as a locally weighted average, using a kernel as a weighting function. For example, one may use the Nadaraya–Watson (NW) estimator $\widehat{m}_y(\vect{z})$ to learn $\mathbb{E}(\mat{y}|\mat{Z})$ as follows:
\begin{align}
\widehat{m}_y(\vect{z})=\frac{\sum_{i=1}^{n} \psi\left(\frac{\vect{z}-\vect{Z}{i}}{h_y}\right)y_{i}}{\sum_{i=1}^{n} \psi\left(\frac{\vect{z}-\vect{Z_{i}}}{h_y}\right)} 
\end{align}
where $h_y$ is a bandwidth whose size is related to the dependence of $\vect{y}$ on $\mat{Z}$ and $\psi(\cdot)$ is the kernel. Popular choices of kernels include uniform, triangle, Epanechnikov, Gaussian, quadratic and cosine. We use the Gaussian kernel, which is also known as the radial basis function (RBF) kernel, throughout the paper.

The second model is the DML with sample-splitting. We take a $K$-fold random partition $\{I_j\}_{j=1}^K$ of the indices $[n]$.

\begin{algorithm}[!htb]
	\centering 
	\caption{Double/Debiased Machine Learning (DML)}
	\begin{algorithmic}
	\STATE{\textbf{Input}: Data matrix $[\mat{D},\mat{Z}]$, label $\vect{y}$} \\
	\FOR{$j\in [K]$}
	\STATE{\quad\quad 1. fit $\vect{y}_{I_j}^C\sim \mat{Z}_{I_j}^C$ via some machine learning method to learn $\mathbb{E}(\vect{y}|\mat{Z})$;}
	\\
	\STATE{\quad\quad 2. fit $\mat{D}_{I_j}^C\sim \mat{Z}_{I_j}^C$ via some machine learning method to learn $\mathbb{E}(\mat{D}|\mat{Z})$;}
	\\
	\STATE{\quad\quad 3. fit $\vect{y}_{I_j}-\mathbb{E}(\vect{y}_{I_j}|\mat{Z}_{I_j})\sim \mat{D}_{I_j}-\mathbb{E}(\mat{D}_{I_j}|\mat{Z}_{I_j})$ via OLS to derive $\bm{\hat{\beta}}$, denoted as $\bm{\hat{\beta}}^{(j)}$;}
	\\
	\ENDFOR
	\STATE{4. aggregate the estimators: $\bm{\hat\beta}=\sum_j \bm{\hat{\beta}}^{(j)}/K.$}
\end{algorithmic}
\label{alg:DML}
\end{algorithm}
We note that the DML does not explicitly predict $\vect{y}$, hence we add an extra step to accomplish this. Denote the estimators in step 1 and step 2 as $m_y^{(j)}$ and $m_D^{(j)}$. We aggregate the estimates of $\mathbb{E}(\vect{y}|\mat{Z})$ and $\mathbb{E}(\mat{D}|\mat{Z})$ from the $K$ estimators and predict 
$$\vect {\hat y}:=\sum_j m_y^{(j)}(\mat{Z})/K+\left(\mat D-\sum_j m_D^{(j)}(\mat{Z})/K\right)\bm{\hat\beta}.$$

We remark that the choice of the machine learning methods to use is flexible. For instance, one may use Lasso if $\mat{Z}$ is high dimensional or use the logistic regression if $\mat{D}$ is categorical.

\section{Proof of Theorem \ref{thm:main_gf}}
\label{AppendixB}
Denote label as $\vect{y}\in \mathbb{R}^{n}$, data as $\mat{X}\in \mathbb{R}^{n\times p}$, the features selected by Lasso as $\vect{D} \in \mathbb{R}^{n\times p_{L}}$ and the rest as $\vect{Z} \in  \mathbb{R}^{n\times p_{N}}$, where $p_L+p_N = p$. Recall that $\mat{W} \in \mathbb{R}^{(p_N)\times m}, \mat{A} \in \mathbb{R}^{m\times (1+p_L)}$ are the weights in first and second layers respectively.
Here we consider a neural network of the following form.
\begin{align*}
F(\mat{W},\vect{A}, \vect{z}) = \frac{1}{\sqrt{m}}\sum_{r=1}^{m}\mat{A}_r \relu{\vect{w}_r^\top \vect{z}}
\end{align*}
where $\vect{w}_r$ and $\mat{A}_r$ are the weights corresponding to the $r$-th neuron in the hidden layer. The input $\vect{z} \in \mat{R}^{p_N}$ and $\relu{\cdot}$ is the ReLU activation function.
To be more clear, we note the $s$-th dimension of $F$ as $F_s(\mat{W},\vect{A}, \vect{z})$, which means for $s=1,2,\cdots,1+p_L$,
\begin{align}
F_s(\mat{W},\mat{A}, \vect{z}) 
=[F(\mat{W},\mat{A}, \vect{z}) ]_s= \frac{1}{\sqrt{m}}\sum_{r=1}^{m}A_{rs} \relu{\vect{w}_r^\top \vect{z}}
\label{eq:network form}
\end{align}
Here $A_{rs}$ is a scalar representing the output weight in the second layer. Given the dataset $\left\{(\vect{Z}_i,\vect{M}_i)\right\}_{i=1}^n$ with the multivariate response $\vect{M}:=[\vect{y}, \vect{D}]$. We aim to minimize 
\begin{align*}
 L(\mat{W},\vect{A}) = \sum_{i=1}^{n}\frac{1}{2}\left\|F(\mat{W},\vect{A},\vect{Z}_i)-\vect{M}_i\right\|_2^2.
\end{align*}
Formally, we consider the gradient flow of the gradient descent defined by:
\begin{align*}
\frac{d\vect{w}_r(t)}{dt} = - \frac{\partial L(\mat{W}(t),\vect{A})}{\partial \vect{w}_r(t)}
\end{align*} for $r=1,\ldots,m$. Simple chain rule gives the MSE loss derivative with respect to each weight vector $\vect{w_r}$ as
\begin{align}
\frac{\partial L(\mat{W},\mat{A})}{\partial \vect{w}_r} 
&= \sum_{i=1}^n\sum_{s=1}^{(1+p_L)}(F_s(\mat{W},\mat{A},\vect{Z}_i) - M_{is})\frac{\partial F_s(\mat{W},\mat{A},\vect{Z}_i)}{\partial \vect{w}_r} \nonumber \\
&= \frac{1}{\sqrt{m}}\sum_{i=1}^n\sum_{s=1}^{(1+p_L)}(F_s(\mat{W},\mat{A},\vect{Z}_i)-{M}_{is})\vect{A}_{rs}\vect{Z}_i\indict\left\{\vect{w}_r^\top \vect{Z}_i \ge 0\right\}, \end{align} 
as the form in \eqref{eq:network form} indicates
\begin{align}
\frac{\partial F_s(\mat{W}(t),\vect{A},\vect{Z}_i)}{\partial \vect{w}_r(t)}&=\frac{1}{\sqrt{m}}\vect{A}_{rs}\vect{Z}_i\indict\left\{\vect{w}_r^\top \vect{Z}_i \ge 0\right\}
\label{eqn:gradient}
\end{align}

Let us shorthand $u_{is}(t)=F_s(W(t), A, Z_i)$. The dynamics of each dimension of one prediction is again given by the chain rule,
\begin{align}
	\frac{d}{dt} u_{is}(t) 
	= &\sum_{r=1}^{m}\langle\frac{\partial F_s(\mat{W}(t),\vect{A},\vect{Z}_i)}{\partial \vect{w}_r(t)},\frac{d \vect{w}_r(t)}{dt} \rangle \nonumber\\
	= &\sum_{r=1}^{m}\langle\frac{\partial F_s(\mat{W},\vect{A},\vect{Z}_i)}{\partial \vect{w}_r}, \sum_{j=1}^n\sum_{h=1}^{(1+p_L)}(M_{jh} - F_h(\mat{W},\mat{A},\vect{Z}_j) )\frac{\partial F_h(\mat{W},\mat{A},\vect{Z}_j)}{\partial \vect{w}_r}\rangle \nonumber \\
	= &\sum_{h=1}^{(1+p_L)}\sum_{j=1}^n(M_{jh} - F_h(\mat{W},\mat{A},\vect{Z}_j))\sum_{r=1}^{m}\langle\frac{\partial F_s(\mat{W},\vect{A},\vect{Z}_i)}{\partial \vect{w}_r}, \frac{\partial F_h(\mat{W},\mat{A},\vect{Z}_j)}{\partial \vect{w}_r}\rangle \nonumber \\
	=& \sum_{h=1}^{(1+p_L)}\sum_{j=1}^{n}(M_{jh}-u_{jh}) (\mat{H}_{sh})_{ij}(t).
	\label{eq:dynamics individual}
\end{align}
Here we define the $n \times n$ matrix $\mat{H}_{sh}(t)$ using \eqref{eqn:gradient} as follows.
\begin{align*}
(\mat{H}_{sh})_{ij}(t) 
=& \sum_{r=1}^{m}\langle\frac{\partial F_s(\mat{W},\vect{A},\vect{Z}_i)}{\partial \vect{w}_r}, \frac{\partial F_h(\mat{W},\mat{A},\vect{Z}_j)}{\partial \vect{w}_r}\rangle\\
=& \frac{1}{m} \vect{Z}_i^\top \vect{Z}_j\sum_{r=1}^m A_{rs}A_{rh} \indict\left\{\vect{Z}_i^\top \vect{w}_r(t) \ge 0, \vect{Z}_j^\top \vect{w}_r(t) \ge 0\right\}.
\end{align*}
Now we can write the dynamics of the predictions \eqref{eq:dynamics individual} in a compact way:
\begin{align}
	\frac{d}{dt} \vect{u}_s(t) &= \sum_{h=1}^{(1+p_L)} \mat{H}_{sh}(t)(\vect{M}_s-\vect{u}_s(t))\nonumber\\
	\frac{d}{dt} (\vect{M}_s- \vect{u}_s(t)) &= -\sum_{h=1}^{(1+p_L)} \mat{H}_{sh}(t)(\vect{M}_s-\vect{u}_s(t))
	\label{eq:dynamics full}
\end{align}

Furthermore, we can rewrite \eqref{eq:dynamics full} by concatenating each dimension of the prediction sequentially and denote the concatenated response in $\mathbb{R}^{n(1+p_L)\times 1}$ as  $\mat{M}_{conc}-\mat{u}_{conc}(t)$. Then the dynamics is 
equivalent to
$$\frac{d}{dt}(\mat{M}_{conc}-\mat{u}_{conc}(t))=\mat{H}_{whole}(\mat{M}_{conc}-\mat{u}_{conc}(t))
$$
with
\begin{align*}
\mat{H}_{whole}
:={\begin{pmatrix}\mathbf {H} _{11}&\mathbf{H}_{12}&\cdots &\mathbf {H}_{1,1+p_L}\\\mathbf {H}_{21}&\mathbf {H} _{22}&\cdots &\mathbf {H}_{2,1+p_L}\\\vdots &\vdots &\ddots &\vdots \\\mathbf {H}_{1+p_L,1}&\mathbf {H}_{1+p_L,2}&\cdots &\mathbf {H} _{1+p_L,1+p_L}\end{pmatrix}}.
\end{align*}

We now consider the NTK matrices corresponding to infinite-width neural network: $$(\mat{H}_{sh})^{\infty}:=\lim_{m \rightarrow \infty}(\mat{H}_{sh}) \text{ \quad and \quad } \mat{H}_{whole}^{\infty}:=\lim_{m \rightarrow \infty}\mat{H}_{whole}.$$

We notice that if $s=h$, since $A_{rs}$ is $\pm 1$, $(\mat{H}_{sh})^{\infty}$ is the same as $\mat{H}^\infty$ in Fact \ref{asmp:main}, which has been proven to be positive definite. If $s \neq h$, we have
\begin{align*}
(\mat{H}_{sh})_{ij}^{\infty}
=&\mathbb{E}_{\mat{A}_s\sim\text{unif}\{-1,1\},\mat{A}_h\sim\text{unif}\{-1,1\},\vect{w}\sim \mathcal{N}(\vect{0}, \vect{I})}(\vect{Z}_i^\top \vect{Z}_j \mat{A}_{s}\mat{A}_{h}\indict\left\{ \vect{w}^\top\vect{Z}_i \ge 0, \vect{w}^\top\vect{Z}_j  \ge 0\right\})\\
=&\vect{Z}_i^\top \vect{Z}_j \mathbb{E}(\mat{A}_{s})\mathbb{E}(\mat{A}_{h})\mathbb{P}_{\vect{w}\sim \mathcal{N}(\vect{0}, \vect{I})}\left\{ \vect{w}^\top\vect{Z}_i \ge 0, \vect{w}^\top\vect{Z}_j  \ge 0\right\})
=0
\end{align*}
From the initialization of $A_{rs}$, we get $\frac{1}{m} \sum_{r=1}^{m} A_{rs}\to\mathbb{E}(\mat{A}_{s}) =0$ as $m\to\infty$ by the law of large numbers. Hence we obtain
\begin{align*}
\mat{H}_{whole}^\infty
:={\begin{pmatrix}\mathbf {H}^\infty&\mathbf{0}&\cdots &\mathbf {0}\\\mathbf {0}&\mathbf {H}^\infty&\cdots &\mathbf {0}\\\vdots &\vdots &\ddots &\vdots \\\mathbf {0}&\mathbf {0}&\cdots &\mathbf {H}^\infty\end{pmatrix}}.
\end{align*}

In summary, we conclude that $\lambda_0 :=\lambda_{\min}(\mat{H}^{\infty})=\lambda_{\min}(\mat{H}_{ss}^{\infty})>0$ and $\lambda_{\min}(\mat{H}_{sh}^{\infty})=0$ for $s\neq h$. Note that the eigenvalues of the block diagonal matrix $\mat{H}_{whole}^\infty$ is the same as those of $\mat{H}^\infty$, hence $\lambda_{\min}(\mat{H}_{whole}^{\infty})=\lambda_0$.

To prove Theorem \ref{thm:main_gf}, we first show that $\mat{H}_{whole}(0)$ is close to $\mat{H}_{whole}^\infty$ and hence is positive definite.
\begin{lemma}
If $m=\Omega\left(\frac{n^2(1+p_L)^{2}}{\lambda_{0}^{2}} \log \left(\frac{n^2(1+p_L)^2}{\delta}\right)\right)$, then we have $\left\|\mathbf{H}_{whole}(0)-\mathbf{H}_{whole}^{\infty}\right\|_{2} \leq \frac{\lambda_{0}}{4}$ and $\lambda_{\min }(\mathbf{H}_{whole}(0)) \geq \frac{3}{4} \lambda_{0}$ with probability of at least $1-\delta$.
\end{lemma}
\begin{proof}
By Hoeffding inequality, we have for each fixed $(i,j)$ pair, we have with probability $1-\delta'$,
\begin{align*}
    |\mat{H}_{whole}(0)-\mat{H}_{whole}^{\infty}|\leq \sqrt{\frac{2}{m}\log \frac{2}{\delta'}}
\end{align*}
For all $(i, j)$, if we set $\delta = n^2(1+p_L)^2\delta'$, then
\begin{align*}
    |\mat{H}_{whole}(0)-\mat{H}_{whole}^{\infty}|\leq \sqrt{\frac{2}{m}\log \frac{2n^2(1+p_L)^2}{\delta}}
\end{align*}
Hence, if $m=\Omega\left(\frac{n^2(1+p_L)^{2}}{\lambda_{0}^{2}} \log \left(\frac{n^2(1+p_L)^2}{\delta}\right)\right)$, then 
\begin{align*}
    \left\|\mathbf{H}_{whole}(0)-\mathbf{H}_{whole}^{\infty}\right\|_{2}^{2} \leq&\left\|\mathbf{H}_{whole}(0)-\mathbf{H}_{whole}^{\infty}\right\|_{F}^{2} \\
    \leq &
    \sum_{i, j}\left|(\mathbf{H}_{whole})_{i j}(0)-(\mathbf{H}_{whole})_{i j}^{\infty}\right|^{2}\\ \leq & \frac{2 n^2(1+p_L)^{2} \log (2n^2(1+p_L)^2 / \delta)}{m}
\end{align*}
which leads to
$$\left\|\mathbf{H}_{whole}(0)-\mathbf{H}_{whole}^{\infty}\right\|_{2}\leq\frac{\lambda_0}{4}.$$
By the standard matrix concentration bound used in \cite{du2018gradient}, $\mat{H}_{whole}(0)$ has a positive least eigenvalue with high probability:
$$\lambda_{\min}(\mat{H}_{whole}(0))\geq\lambda_{\min}(\mat{H}_{whole}^\infty)-\frac{\lambda_0}{4}.$$
\end{proof}

\begin{lemma}
If $\vect{w}_r$ are i.i.d generated from $\mathcal{N}(\vect{0},\mat{I})$ for $r \in [m]$, and $\left\|\mathbf{w}_{r}(0)-\mathbf{w}_{r}\right\|_{2} \leq \frac{c \delta \lambda_{0}}{n^2(1+p_L)^{2}} =: R$ for some small positive constant $c$, then the following holds with probability at least $1-\delta$ we have $\|\mat{H}_{whole}-\mat{H}_{whole}(0)\|_2<\frac{\lambda_0}{4}$ and $\lambda_{\min}(\mat{H}_{whole})>\frac{\lambda_0}{2}$.
\label{lem:B2}
\end{lemma}
\begin{proof}
First we set 
\begin{align*}
    A_{r}^{i}=\left\{\exists \mathbf{w}:\left\|\mathbf{w}-\mathbf{w}_{r}(0)\right\| \leq R, \mathbb{I}\left\{\mathbf{x}_{i}^{\top} \mathbf{w}_{r}(0) \geq 0\right\} \neq \mathbb{I}\left\{\mathbf{x}_{i}^{\top} \mathbf{w} \geq 0\right\}\right\}
\end{align*}
Since $\vect{w}_r \sim N(\vect{0},\mat{I})$, By the anti-concentration inequality of Gaussian we have
\begin{align*}
    P\left(A_{r}^{i}\right)=P_{w \sim N(0,1)}(|w|<R) \leq \frac{2 R}{\sqrt{2 \pi}}
\end{align*}
Therefore, for any $\vect{w}_r$ that satisfies the assumption, we have 
\begin{align*}
 &\mathbb{E}\left[\left|(\mat{H}_{sh})_{i j}(0)-(\mat{H}_{sh})_{i j}\right|\right] \\
 =& \mathbb{E}\left[\frac{1}{m} \biggm | \vect{Z}_{i}^{\top} \vect{Z}_{j}A_{rs}A_{rh} \sum_{r=1}^{m}\left(\mathbb{I}\left\{\vect{w}_{r}(0)^{\top} \vect{Z}_{i} \geq 0, \vect{w}_{r}(0)^{\top} \vect{Z}_{j} \geq 0\right\}-\mathbb{I}\left\{\vect{w}_{r}^{\top} \vect{Z}_{i} \geq 0, \vect{w}_{r}^{\top} \vect{Z}_{j} \geq 0\right\}\right) \right].\\ 
 \leq & \frac{1}{m} \sum_{r=1}^{m} \mathbb{E}\left[\mathbb{I}\left\{A_{r}^{i} \cup A_{r}^{j}\right\}\right] \leq \frac{4 R}{\sqrt{2 \pi}} 
\end{align*}
Taking the sum over $(i, j)$ we obtain
\begin{align*}
    \mathbb{E}\left(\Sigma_{(i,j)=(1,1)}^{(n(1+p_L),n(1+p_L))}\|(\mat{H}_{sh})_{ij}-(\mat{H}_{sh})_{ij}(0)\|\right) \leq \frac{4 (n(1+p_L))^2 R}{\sqrt{2 \pi}}
\end{align*}
Furthermore, by Markov's inequality
\begin{align*}
    \mathbb{E}\left(\Sigma_{(i,j)=(1,1)}^{(n(1+p_L),n(1+p_L))}\|(\mat{H}_{sh})_{ij}-(\mat{H}_{sh})_{ij}(0)\|\right) \leq \frac{4 (n(1+p_L))^2 R}{\sqrt{2 \pi}\delta}
\end{align*}
with probability $1-\delta$. By  matrix perturbation theory we get the bound of $\mat{H}_{whole}$ with initialization,
\begin{align*}
    \|\mat{H}_{whole}-\mat{H}_{whole}(0)\|_{2} \leq& \|\mat{H}_{whole}-\mat{H}_{whole}(0)\|_{F} \\
    \leq& \sum_{(s, h)=(1, 1)}^{(1+p_L,1+p_L)}\sum_{(i, j)=(1,1)}^{(n, n)}\left|(\mat{H}_{sh})_{i j}-(\mat{H}_{sh})_{i j}(0)\right| \\
    \leq& \frac{4 n^2 (1+p_L)^{2} R}{\sqrt{2 \pi} \delta}
\end{align*}
Plugging in $R$,
\begin{align*}
    \lambda_{\min }(\mat{H}_{whole}) \geq \lambda_{\min }(\mat{H}_{whole}(0))-\frac{4 n^2(1+p_L)^{2} R}{\sqrt{2 \pi} \delta} > \frac{\lambda_{0}}{2}
\end{align*}
\end{proof}

\begin{lemma}
Assume $0\leq t_1 \leq t$ and $\lambda_{\min}(\mat{H}_{whole}(t_1))\geq \frac{\lambda_0}{2}$. Then we have $\|\mat{M}_{s}-\mat{u}_{s}(t)\|_2^2 \leq \exp(-\lambda_0(1+p_L) t) \|\mat{M}_{s}-\mat{u}_{s}(0)\|_2^2$ and $\left\|\vect{w}_{r}(t)-\vect{w}_{r}(0)\right\|_{2} \leq \frac{\sqrt{n}}{\sqrt{m} \lambda_{0}} \sum_{h=1}^{1+p_L}\|\vect{M}_{h}-\vect{u}_{h}(0)\|_{2} =: R^{\prime}$.
\label{lem:B3}
\end{lemma}
\begin{proof}
\begin{align*}
\frac{d}{dt}(\mat{M}_{conc}-\mat{u}_{conc}(t))=\mat{H}_{whole}(\mat{M}_{conc}-\mat{u}_{conc}(t))
\end{align*}
Hence, 
\begin{align*}
	\frac{d}{dt} \|\vect{M}_{conc} - \vect{u}_{conc}(t)\|_2^2 
	=& -2 (\vect{M}_{conc}-\vect{u}_{conc}(t))^\top \mat{H}_{whole}(\vect{M}_{conc}-\vect{u}_{conc}(t))\\
	\leq & -\lambda_0 \|\mat{M}_{conc}-\mat{u}_{conc}\|_2^2.
\end{align*}
Therefore, we can bound the loss
\begin{align*}
    \|\vect{M}_{conc} - \vect{u}_{conc}(t)\|_2^2\leq \exp(-\lambda_0t)\|\vect{M}_{conc} - \vect{u}_{conc}(0)\|_2^2.
\end{align*}
Hence for each $s$,
\begin{align*}
    \|\vect{M}_s - \vect{u}_s(t)\|_2^2\leq \exp(-\lambda_0t)\|\vect{M}_s - \vect{u}_s(0)\|_2^2,
\end{align*}
which proves the exponentially fast convergence of each dimension and that all dimensions evolve under the same dynamics.

For $0\leq t_1 \leq t$, 
\begin{align*}
\left\|\frac{d}{d t_1} \vect{w}_{r}(t_1)\right\|_{2} &=\left\|\frac{1}{\sqrt{m}}\sum_{i=1}^{n}\sum_{h=1}^{1+p_L}\left(M_{ih}-u_{ih}\right) A_{rs} \vect{Z}_{i} \indict\left\{\vect{w}_{r}(t_1)^{\top} \vect{Z}_{i} \geq 0\right\}\right\|_{2} \\ 
& \leq \frac{1}{\sqrt{m}} \sum_{i=1}^{n}\sum_{h=1}^{1+p_L}\left|M_{ih}-u_{ih}(t_1)\right| \\
&\leq \frac{\sqrt{n}}{\sqrt{m}}\sum_{h=1}^{1+p_L}\|\vect{M}_h-\vect{u}_h(t_1)\|_{2} \\
&\leq \frac{\sqrt{n}}{\sqrt{m}}\sum_{h=1}^{1+p_L}\exp \left(-\lambda_{0}t_1/2 \right)\|\vect{M}_h-\vect{u}_h(0)\|_{2} 
\end{align*}
In the end, 
\begin{align*}
    \left\|\mathbf{w}_{r}(t)-\mathbf{w}_{r}(0)\right\|_{2} 
    \leq \int_{0}^{t}\left\|\frac{d}{d s} \mathbf{w}_{r}(s)\right\|_{2} d s 
    \leq \frac{\sqrt{n}}{\sqrt{m} \lambda_{0}}\sum_{h=1}^{1+p_L}\|\mathbf{M}_h-\mathbf{u}_h(0)\|_{2}
\end{align*}
\end{proof}

\begin{lemma}
If $R' < R$, then for all $t\geq0$, we have $\lambda_{\min}(\mat{H}_{whole}(t))\geq \frac{1}{2}\lambda_0$, for all $r\in [m]$, $\|\vect{w}_r(t)-\vect{w}_r(0)\|_2 \leq R'$ and for all $s\in[1+p_L]$, $\|\mat{M}_{s}-\vect{u}_{s}(t)\|_2^2 \leq \exp(-\lambda_0t)\|\mat{M}_{s}-\vect{u}_{s}(0)\|_2^2$.
\end{lemma}
\begin{proof}
This lemma takes the same form as \cite[Lemma 3.4]{du2018gradient}: given Lemma \ref{lem:B2} and Lemma \ref{lem:B3}, the result clearly follows and we refer the interested readers to \cite{du2018gradient} for details. 
\end{proof}
It is sufficient to show that $R'<R$ is equivalent to $m=\Omega\left(\frac{n^{5}\sum_{h=1}^{1+p_L}\|\mat{M}_h-\mat{u}_h(0)\|_{2}^{2}}{\lambda_{0}^{4} \delta^{2}}\right)$. We bound
\begin{align*}
&\mathbb{E}\left[\|\mat{M}_{s}-\mat{u}_{s}(0)\|_{2}^{2}\right]\\
=&\sum_{i=1}^{n}\left(M_{is}^{2}+M_{is} \mathbb{E}\left[F_s\left(\mat{W}(0), \mat{A}, \vect{Z}_{i}\right)\right]+\mathbb{E}\left[F_s\left(\mat{W}(0), \mat{A}, \vect{Z}_{i}\right)^{2}\right]\right)\\
=&\sum_{i=1}^{n}\left(M_{is}^{2}+1\right)=O(n)
\end{align*}
Using the Markov’s inequality, we have $\|\mat{M}_s-\mat{u}_s(0)\|_2^2 = O(\frac{n}{\delta})$ with probability at least $1-\delta$.
Hence we have $m = \Omega \left( \frac{(1+p_L)^5n^6}{\delta^3 \lambda_0^4}\right)$.
\section{Proof of Theorem \ref{prop:consistency}}
\label{app:proof of prop}
\begin{proof}
Recall that in \eqref{eq:truebeta} we have
\begin{align*}
\mathcal{Y}=\mathbf y-\mathbb{E}(\mathbf y|\mathbf Z)=(\mathbf D-\mathbb{E}(\mathbf D|\mathbf Z))\bm\beta+\bm\epsilon=\mathcal{X}\bm\beta+\bm\epsilon.
\end{align*}
Unfortunately we only observe data with noises
\begin{align*}
\widetilde{\mathcal{Y}}
=\mathcal{Y}+\bm{\epsilon}_{Y}
=\mathcal{X}\bm\beta+\bm\epsilon+\bm\epsilon_Y
=\widetilde{\mathcal{X}}\bm\beta-\bm\epsilon_X\bm\beta+\bm\epsilon+\bm\epsilon_Y
\end{align*}
Hence we can derive \eqref{beta_hat} as the following:
\begin{align*}
\bm{\tilde\beta}=& (\widetilde{\mathcal{X}}^\top\widetilde{\mathcal{X}})^{-1}\widetilde{\mathcal{X}}^\top\widetilde{\mathcal{Y}}\\
=& (\widetilde{\mathcal{X}}^\top\widetilde{\mathcal{X}})^{-1}\widetilde{\mathcal{X}}^\top\left(\widetilde{\mathcal{X}}\bm\beta-\bm\epsilon_X\bm\beta+\bm\epsilon+\bm\epsilon_Y\right)\\
=&\bm\beta-(\widetilde{\mathcal{X}}^\top\widetilde{\mathcal{X}})^{-1}\widetilde{\mathcal{X}}^\top\bm\epsilon_X\bm\beta+(\widetilde{\mathcal{X}}^\top\widetilde{\mathcal{X}})^{-1}\widetilde{\mathcal{X}}^\top(\bm\epsilon+\bm\epsilon_Y)\\
=&\left(1-(\widetilde{\mathcal{X}}^\top\widetilde{\mathcal{X}})^{-1}\widetilde{\mathcal{X}}^\top\bm\epsilon_X\right)\bm\beta+(\widetilde{\mathcal{X}}^\top\widetilde{\mathcal{X}})^{-1}\widetilde{\mathcal{X}}^\top(\bm\epsilon_Y+\bm\epsilon)\\
=&\left(1-\left(\frac{\widetilde{\mathcal{X}}^\top\widetilde{\mathcal{X}}}{n}\right)^{-1}\frac{\widetilde{\mathcal{X}}^\top\bm\epsilon_X}{n}\right)\bm\beta+\left(\frac{\widetilde{\mathcal{X}}^\top\widetilde{\mathcal{X}}}{n}\right)^{-1}\frac{\widetilde{\mathcal{X}}^\top(\bm\epsilon_Y+\bm\epsilon)}{n}.
\end{align*}
We start with investigating the bias. Since $\bm\epsilon, \bm\epsilon_Y$ and $\bm\epsilon_X$ are independent of other variables, we have $$\frac{\widetilde{\mathcal{X}}^\top(\bm\epsilon_Y+\bm\epsilon)}{n}\to 0.$$
In addition, $$\frac{\widetilde{\mathcal{X}}^\top\bm\epsilon_X}{n}=\frac{\mathcal{X}^\top\bm\epsilon_X+\|\bm\epsilon_X\|_2^2}{n}\to\mathbb{E}(\epsilon_X^2)=\sigma_X^2\mathbf{I}_{p_L}.$$
Next, we denote the convergence in probability as plim and look at
\begin{align*}
\frac{\widetilde{\mathcal{X}}^\top\widetilde{\mathcal{X}}}{n}\to\text{plim}\frac{{\mathcal{X}}^\top{\mathcal{X}}}{n}+\text{plim}\frac{\bm\epsilon_X^\top\bm\epsilon_X}{n}=\mathcal{Q}+\sigma_X^2\mathbf{I}_{p_L}
\end{align*}
where $\mathcal{Q}:=\text{plim}\frac{{\mathcal{X}}^\top{\mathcal{X}}}{n}$ exists by the law of large numbers because $\mathcal{X}$ is i.i.d. in rows. Therefore, we obtain
$$\bm{\widetilde\beta}\to\Big(\mathbf{I}-\sigma_X^2\left(\mathcal{Q}+\sigma_X^2\mathbf{I}_{p_L}\right)^{-1}\Big)\bm\beta:=(\mathbf{I}-\mathbf{R})\bm\beta$$ 
where $\mathbf{R}:=\sigma_X^2\left(\mathcal{Q}+\sigma_X^2\mathbf{I}_{p_L}\right)^{-1}=\sigma_X^2\text{plim}\left(\frac{\widetilde{\mathcal{X}}^\top\widetilde{\mathcal{X}}}{n}\right)^{-1}$. We claim that $\bm{\widetilde\beta}$ is a consistent estimator of $\bm\beta$ if and only if $m_D$ is consistently approximated (meaning $\sigma_X^2=0$). If $m_D$ is not consistently approximated, we can modify the estimator via $\mathbf{R}\in\mathbb{R}^{p_L\times p_L}$ and the new estimator $\left(\mathbf{I}-\mathbf{R}\right)^{-1}\bm{\widetilde\beta}$ which is a consistent estimator of $\bm\beta$.

To establish the $\sqrt{n}$-consistency, let us consider the asymptotic distribution of the OLS estimator. Multiplying $\sqrt{n}$ on $\bm{\widetilde\beta}-(\mathbf{I}-\mathbf{R})\bm\beta$ and taking to the limit, we have
\begin{align*}
\sqrt{n}\left(\bm{\widetilde\beta}-(\mathbf{I}-\mathbf{R})\bm\beta\right)
\to&\text{plim}\left(\frac{\widetilde{\mathcal{X}}^\top\widetilde{\mathcal{X}}}{n}\right)^{-1}\frac{\widetilde{\mathcal{X}}^\top(\bm\epsilon_Y+\bm\epsilon)}{\sqrt{n}}
\end{align*}
and
\begin{equation*}
n\left(\bm{\widetilde\beta}-(\mathbf{I}-\mathbf{R})\bm\beta\right)\left(\bm{\widetilde\beta}-(\mathbf{I}-\mathbf{R})\bm\beta\right)^\top
\to
\text{plim}\left(\frac{\widetilde{\mathcal{X}}^\top\widetilde{\mathcal{X}}}{n}\right)^{-1}\left(\frac{\widetilde{\mathcal{X}}^\top(\bm\epsilon_Y+\bm\epsilon)}{\sqrt{n}}\right)\left(\frac{\widetilde{\mathcal{X}}^\top(\bm\epsilon_Y+\bm\epsilon)}{\sqrt{n}}\right)^\top\left(\frac{\widetilde{\mathcal{X}}^\top\widetilde{\mathcal{X}}}{n}\right)^{-1}
\end{equation*}
Making use of $\text{plim}\left(\frac{\widetilde{\mathcal{X}}^\top\widetilde{\mathcal{X}}}{n}\right)^{-1}=\mathbf{R}/\sigma_X^2$ and after some calculation, we arrive at
\begin{align*}
n\left(\bm{\widetilde\beta}-(\mathbf{I}-\mathbf{R})\bm\beta\right)\left(\bm{\widetilde\beta}-(\mathbf{I}-\mathbf{R})\bm\beta\right)^\top
\to\frac{\left(\sigma_\epsilon^2+\sigma_Y^2\right)\mathbf{R}}{\sigma_X^2}
\end{align*}
Thus we complete the proof. In addition, the asymptotic normality can be easily derived by applying the central limit theorem and the Slutsky's theorem.
\end{proof}

\section{Details of Experiments and Extra Application}
\subsection{Data generation for Table \ref{table1:compare with PLMs}}
\label{table1:data}
In Table \ref{table1:compare with PLMs} and Figure \ref{fig:NTK loss}, we generate $\mat{D}$ and $\vect{y}$ using $\mat{Z}$, which is generated by a multivariate standard normal distribution:
\begin{align*}
   & \mat{D} = 50\sum_{j=1}^{10}\sin\vect{Z}_j+\mathcal{N}(\vect{0}, \vect{I})\\
   & \vect{y} = \mat{D}+\sum_{j=1}^{10}\cosh{\mat{Z}_j}+\mathcal{N}(\vect{0}, \vect{I})
\end{align*}
Here $\mat{X}\in \mathbb{R}^{10000 \times 11} = [\mat{D}, \mat{Z}]$,
$\mat{D}\in \mathbb{R}^{n \times 1}$, $\mat{Z}\in \mathbb{R}^{10000 \times 10}$. Though the problem here is non-linear (in fact partially linear), unlike in the debiasing setting in Table \ref{table2:debiasing}, it is fair since all other PLMs work on such problem. Our goal is to demonstrate that PLM-NN is a strong candidate in the PLM family. Here we consider a univaraite $\mat{D}$, so that DML can apply methods including vanilla Lasso to solve this problem. In addition, we design $\mat{D}$ that is dependent on $\mat{Z}$ so that the NW kernel can use a finite bandwidth.

For Table \ref{table1:compare with PLMs}, we train the two-layer neural network with Adam, width 10000 and learning rate 0.0002. For  and Figure \ref{fig:NTK loss}, we train the same network with full-batch gradient descent and learning rate 0.01.

\subsection{Data generation for Table \ref{table2:debiasing}}
\label{table2:data}
In Table \ref{table2:debiasing} and Figure \ref{fig:low dim}, we have
\begin{align*}
    \vect{y}=\mat{X}\bm{\theta}+\mathcal{N}(\vect{0}, \vect{I})
    =\mat{D}\bm{\beta}+\mat{Z}\bm{\gamma}+\mathcal{N}(\vect{0}, \vect{I})
\end{align*}
where $\bm \theta=[1,1,\dots, 0]$ with the first $k$ entries as ones and the rest as zeros.
$\mat{X}$ is generated from a multivariate standard normal distribution. In the second equality, we have $\mat{D}$ as the columns from $\mat{X}$ that are selected by Lasso. The Lasso penalties are chosen specifically to select a moderate number of features so that OLS is available. Here $\mat{D} \in \mathbb{R}^{1000\times k_{Lasso}}, \mat{Z} \in \mathbb{R}^{1000\times(p-k_{Lasso})}$, with $k_{Lasso}=\#\{j:[\bm{\hat\beta}_{Lasso}]_j\neq 0\}$.
We note that when the dimension is relatively high or when the sparsity is relatively large, it is impossible to achieve full power (or true positive rate, or precision), no matter how one carefully tunes the penalty of Lasso. In other words, Lasso must select in some null signals. This phenomenon, known as the Donoho-Tanner phase transition, motivates our experimental settings that consider high dimension and/or high sparsity. For Table \ref{table2:debiasing} and Figure \ref{fig:low dim}, we train the two-layer neural network with Adam, width 1000 and learning rate 0.0002.

\subsection{Complementary Experiments to Table \ref{table1:compare with PLMs}}
In this section, we conduct experiments on more complicated synthetic data to complement Table \ref{table1:compare with PLMs} and Figure \ref{fig:compare with PLMs}. Our new setting is $\mathbf{D}=\sin(\mathbf{T}_1)+\log(\mathbf{T}_2+1)+\frac{1}{1+\mathbf{T}_3}+\max(0,\mathbf{T}_4)+\mathbf{T}_5^2+\mathcal{N}(0,\mathbf{I})$ and $\mathbf{y}=\mathbf{D}+\textup{cosh}(\mathbf{T}_1)+\mathbf{T}_2+\mathbf{T}_3\times\mathbf{T}_4+\mathcal{N}(0,\mathbf{I})$. It is clear that PLM-NN again demonstrates high level of performance in terms of consistency and accuracy, with linear convergence in the training dynamics.
\begin{table}[!htb]
\centering
\begin{tabular}{|cccc|}
\hline 
PLMs & Est MSE & Train& Test
\\
& ($10^{-5}$)&MSE&MSE
\\
\hline
\toprule
PLM-NN & 4.00 & 3.81 & 3.87 \\
PLM-NW & 1.26 & 3.51 & 3.66\\
DML Lasso & 0.38 & 4.17 & 3.99 \\
DML DT & 240 & 8.41 & 15.51\\
\bottomrule
\end{tabular}
\captionof{table}{Comparison of PLMs in 50 independent runs with new data generation. Here PLMNW denotes the PLM using NW kernel, DT denotes decision trees at depth of 2.}
\label{table4:complex}
\end{table}

\vspace{-0.7cm}
\begin{figure}[!htb]
  \centering
  \includegraphics[width=0.4\linewidth]{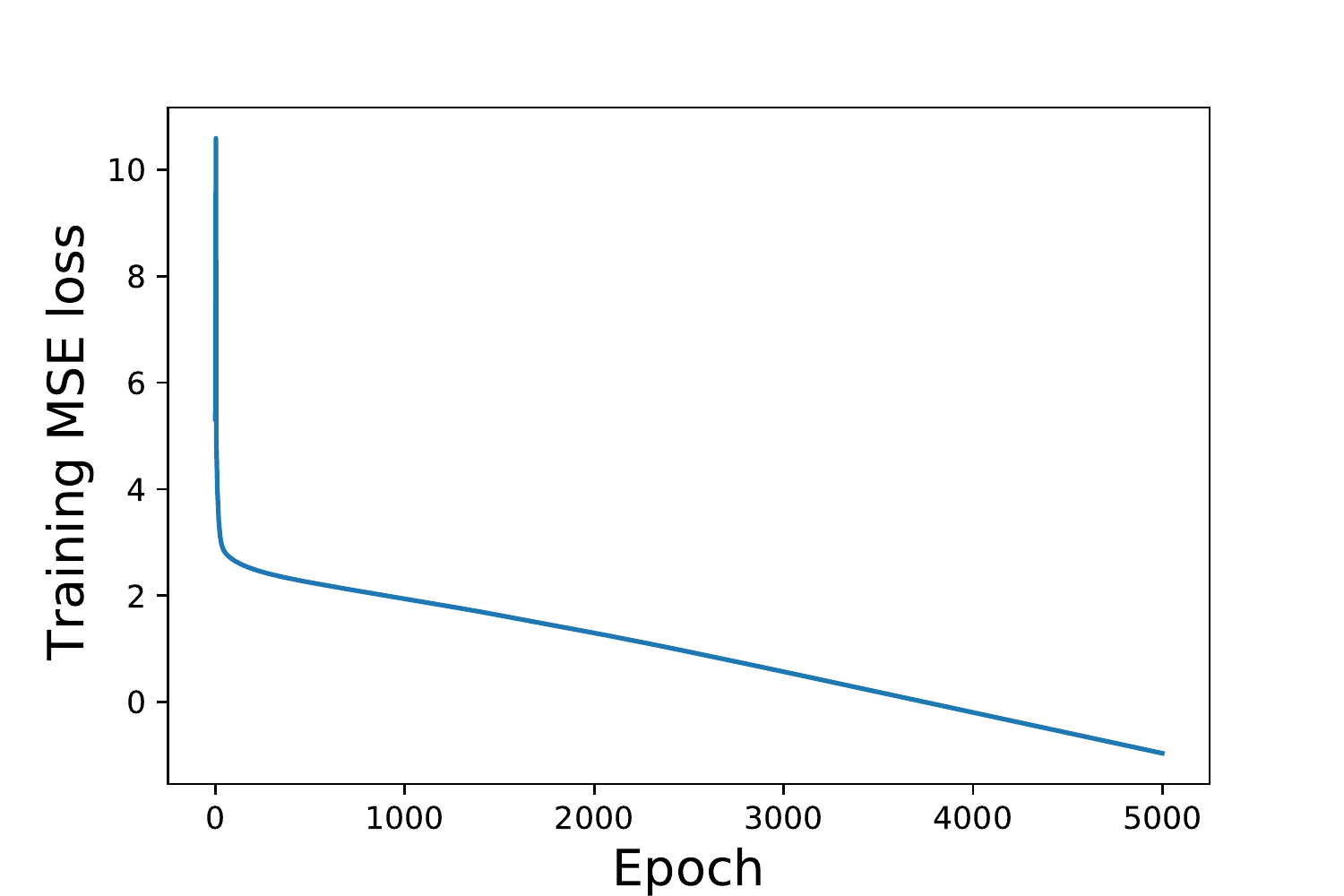}
  \captionof{figure}{Same setting as Table \ref{table4:complex} except $n=100$ and $\mathbf{Z}$ is normalized. The loss is in logarithmic scale.}
\end{figure}

\subsection{Adversarial attack on tabular data}
\label{AppendixD2}
Using the knowledge of the impact of a feature on the output, we can design adversarial samples to attack a potentially strong trained model.
The dataset we use is Default of Credit Card Clients Dataset from UCI. The dataset includes the information on default payments, demographic factors, credit data, history of payment, and bill statements of credit card clients in Taiwan from April 2005 to September 2005. The whole dataset includes 25 columns and a sample size of 30000.
We'll conduct null value imputation by association rules between categorical variables.

We preprocess the dataset in the following way: 
we use the known columns to train logistic regressions and to predict the missing values in the columns 'MARRIAGE' and 'EDUCATION'. Finally we conduct the standard scaling for $\mat{X}$ after one-hot encoding all the categorical columns.

We attack a ReLU activated MLP of three layers with the DebiNet and the traditional OLS. Our attacking mechanism takes the column with the maximum coefficient magnitude calculated by OLS or some feature selection methods. Once we find the targeted column, say the $j$-th column, we perturb the values towards the maximum or minimum value within the column, depending on the sign of $\hat\beta_j$.

\begin{table}[!htb]
\centering
\begin{tabular}{|cccc|}
\hline 
Methods & Val loss& Accuracy & AUC score \\
\hline 
\toprule
MLP Baseline&0.43&0.824&0.77\\
OLS Attack&0.39 &0.840&0.69\\
DebiNet Attack &3.95&0.008&0.02\\
\bottomrule
\end{tabular}
\vspace{0.2cm}
\caption{Adversarial attack on tabular dataset: the Taiwan credit.}
\label{table4}
\end{table}

The DebiNet is equipped with Lasso or Elastic net. The attacks show the same performance because the most influencing feature given by all three feature selection methods is identical. We note that the feature under perturbation is the first repayment status, indicating that not paying back the first credit bill has a large impact on the credibility of a client.
The perturbing columns selected by OLS and DebiNet are different. OLS fails to select what DebiNet selects because some features have infinity coefficients, which might caused by model overfitting and consequently large variances of the parameter estimators. The validation loss, accuracy and AUC score in Table \ref{table4} show that OLS attack is not effective as DebiNet.
\section{General Losses, Activation Functions and Other Optimizers}
\label{AppendixE}
We demonstrate the convergence (in log-scale) of two-layer, fully connected, multivariate output neural networks with the same input distribution and weight initialization as in Figure \ref{fig:NTK loss}. The baseline is the ReLU activation, full-batch gradient descent and MSE loss and we change only one element in the baseline at a time. 

Similar to the univariate output case, we observe that with sufficiently wide hidden layer and sufficiently small learning rate, the neural networks may converge to the global minimum at a linear rate with different optimizers, losses and activation functions. In particular, we note that the performance in this section is not yet supported by theory and thus suggests that NTK theory may be richer than it currently is. We highlight that advanced tools from linear algebra are required to analyze the positive definiteness of the NTK matrices incurred by other activation functions. Moreover, different optimizers lead to different gradient flows and, together with different losses, to different matrix ordinary differential equations. It would be interesting to investigate the effects of different dynamics on the convergence under the NTK framework.

\subsection{Optimizers besides Gradient Descent}
\begin{figure}[!htbp]
    \centering
    \hspace{-0.5cm}
    \includegraphics[width=0.35\textwidth]{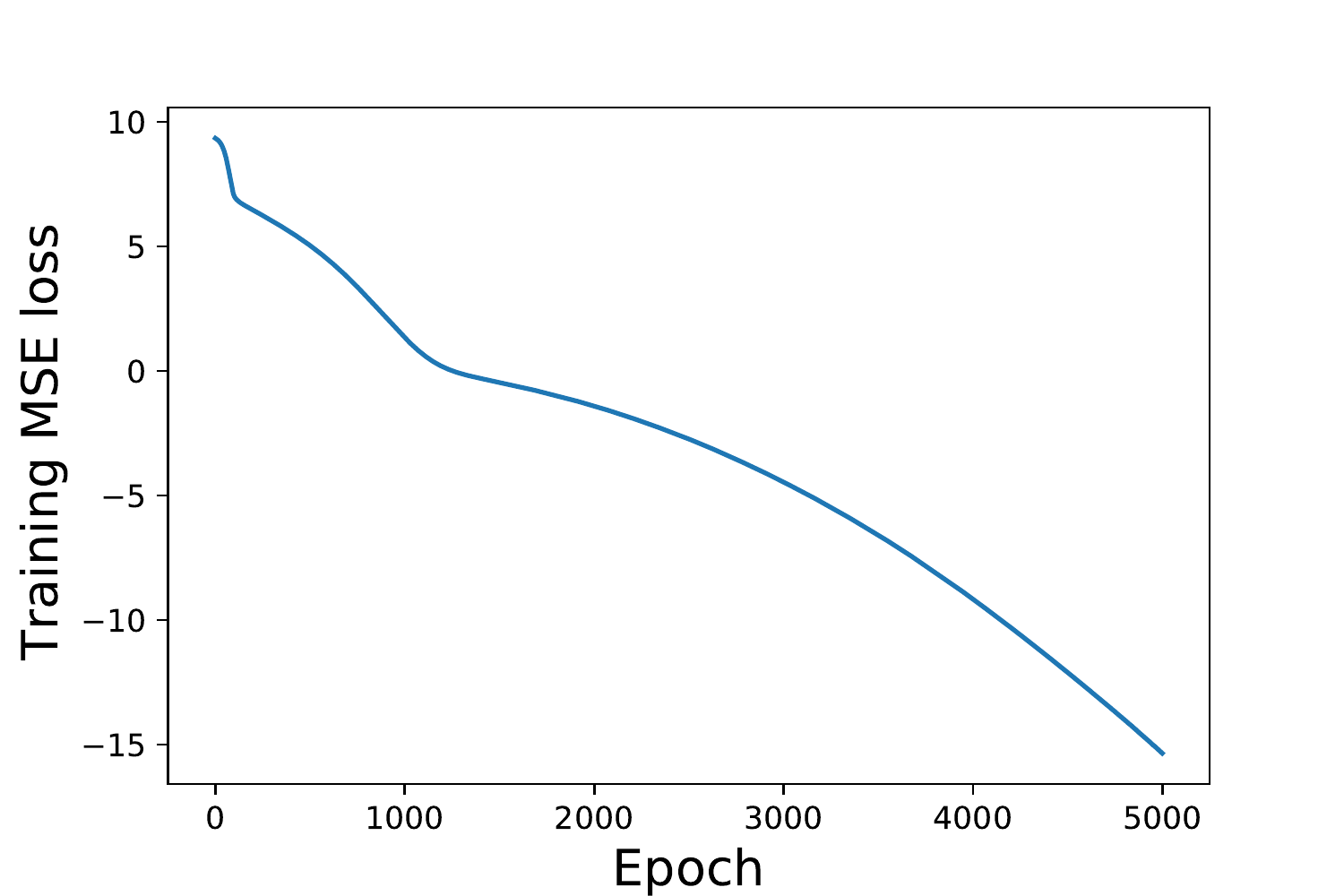}
    \hspace{-0.5cm}
    \includegraphics[width=0.35\textwidth]{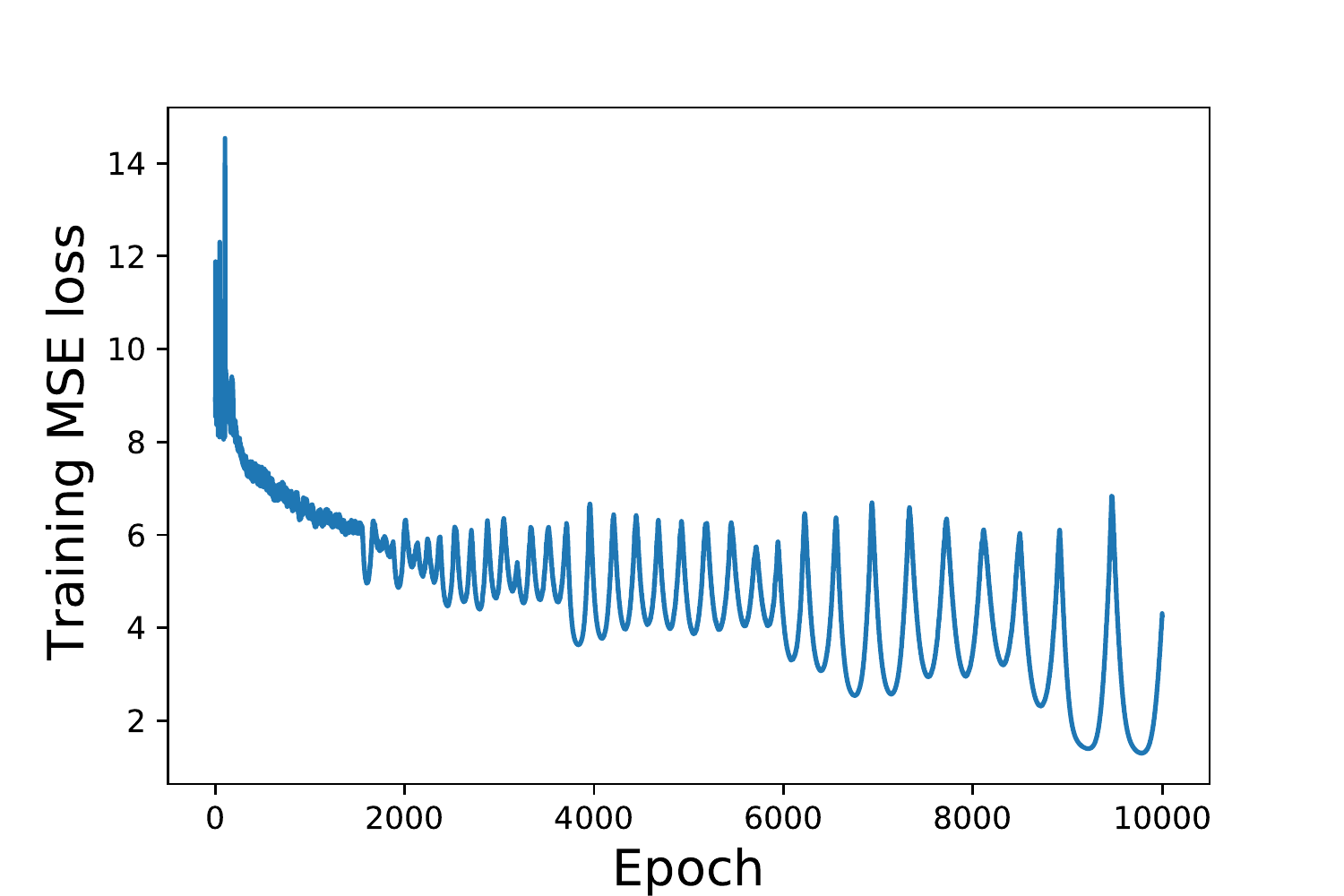}
    \hspace{-0.5cm}
    \includegraphics[width=0.35\textwidth]{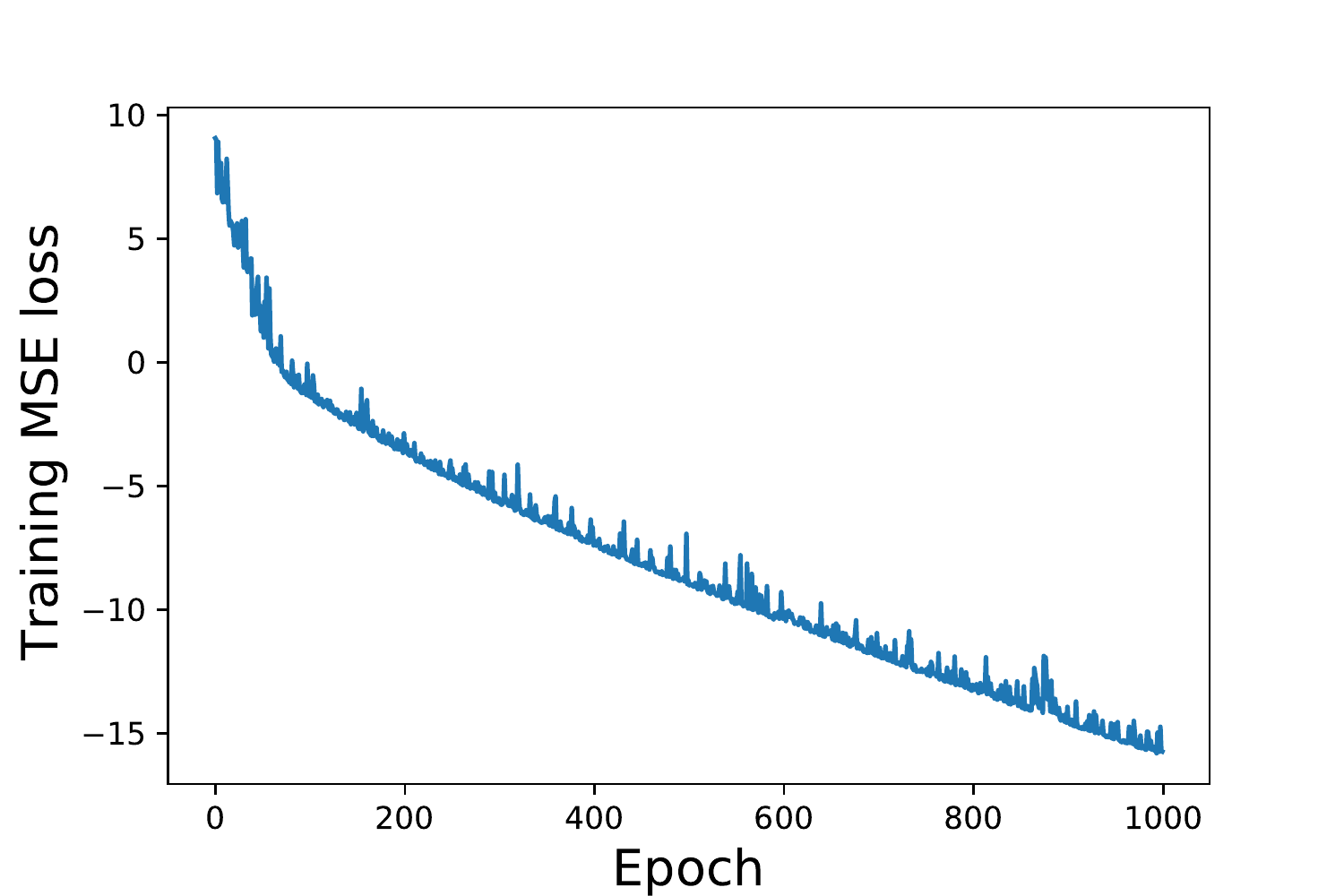}
    \hspace{-0.5cm}
    \caption{Left: Adam with learning rate 0.001 and full batch size. Middle: Nesterov-accelarated gradient descent with learning rate 0.01. Right: SGD with learning rate 0.01 and batch size 8.}
    \label{fig:optimizers}
\end{figure}
In Figure \ref{fig:optimizers}, Adam converges to zero loss exponentially fast but the Nesterov-accelarated gradient descent seems not to. This shows that momentum has a significant impact on the convergence. It is also interesting to note that SGD converges much faster than the gradient descent, at a linear rate. The plot shows some fluctuations due to the randomness of sub-sampling and coincides with the dynamics described in the univariate output case.

\newpage
\subsection{General Losses and Activation Functions}
\begin{figure}[!htb]
    \centering
    \hspace{-0.5cm}
    \includegraphics[width=0.35\textwidth]{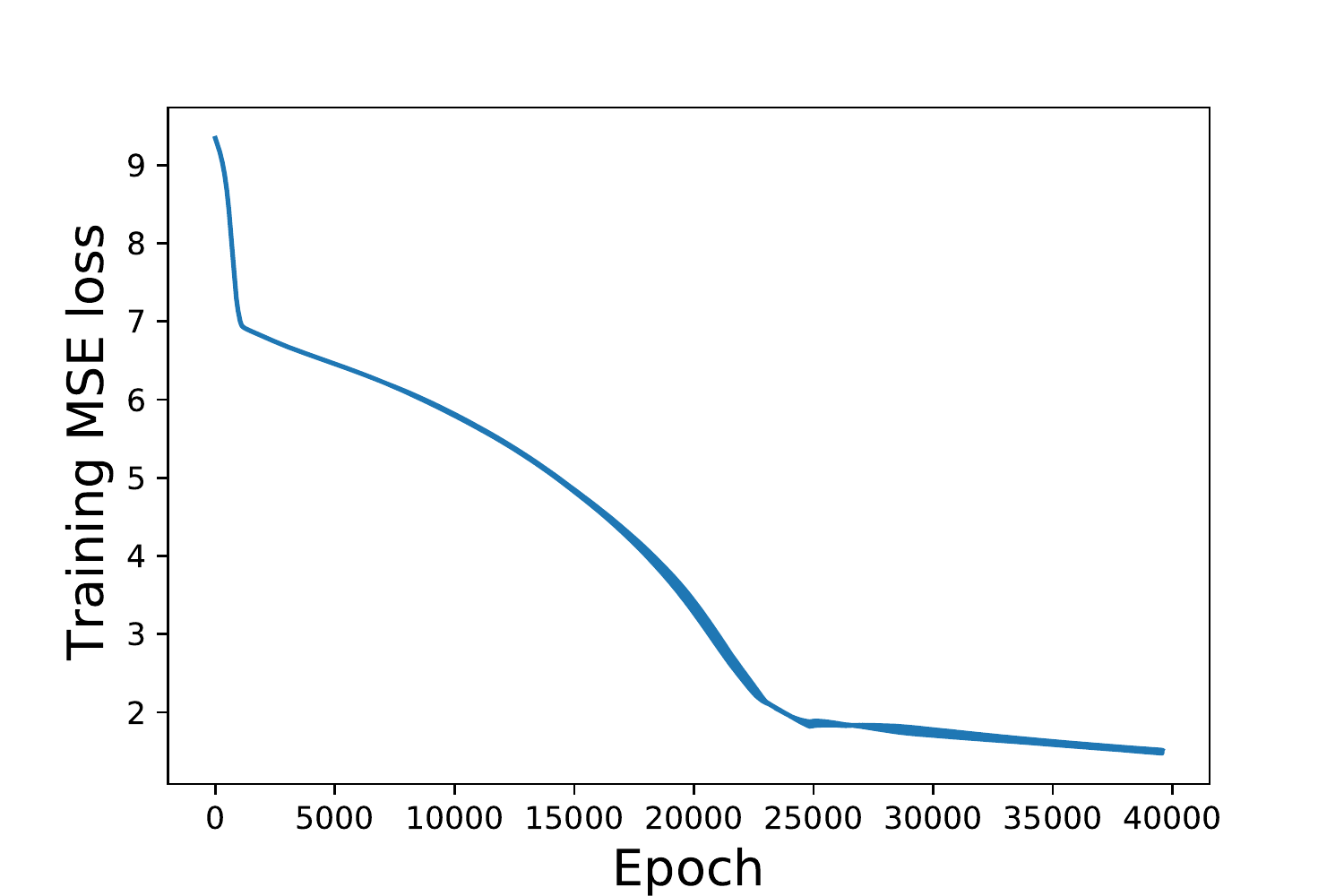}
    \hspace{-0.5cm}
    \includegraphics[width=0.35\textwidth]{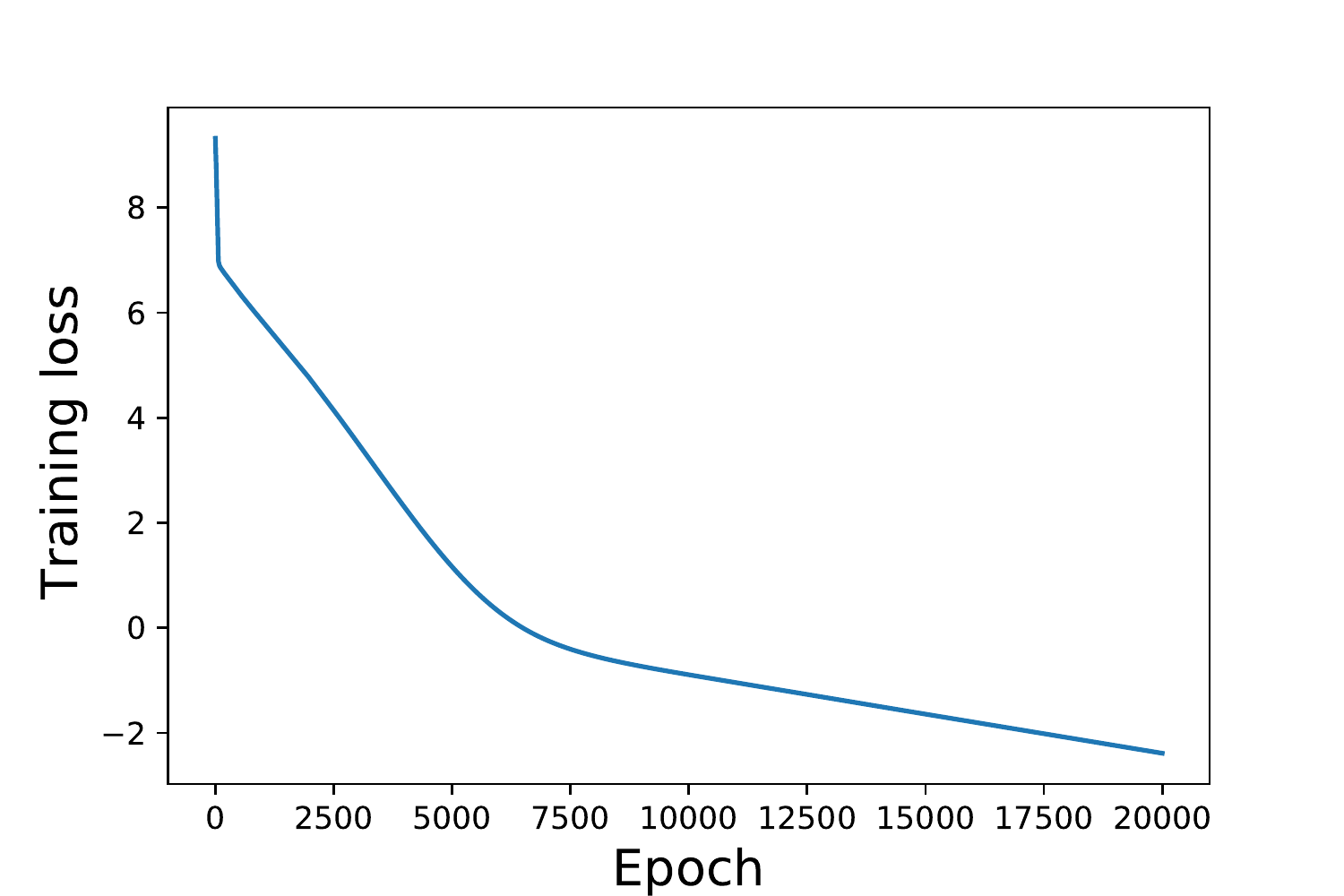}
    \hspace{-0.5cm}
    \includegraphics[width=0.35\textwidth]{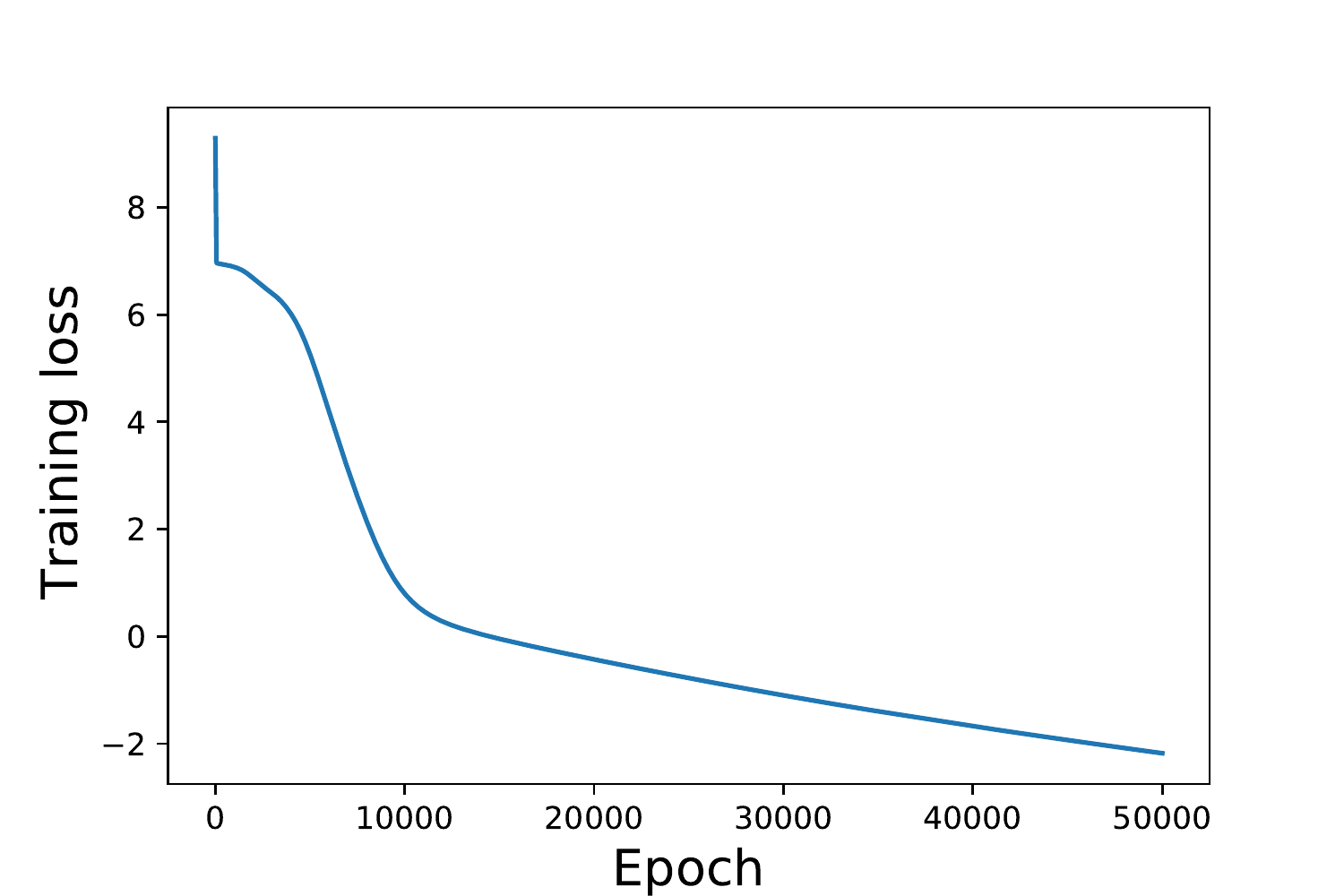}
    \hspace{-0.5cm}
    \caption{Left: Huber loss. Middle: Leaky ReLU activation. Right: Tanh activation.}
    \label{fig:loss and activation}
\end{figure}
In Figure \ref{fig:loss and activation}, we empirically illustrate that loss functions play  a key role in the training dynamics as they directly affect the dynamics. While the activation functions may have weaker effects on whether the dynamics converge to the global minimum, they can influence the rates at which the convergence takes place.
\end{document}